\documentclass[twoside,11pt]{article}

\usepackage{blindtext}

\usepackage[preprint]{jmlr2e}
\usepackage{xcolor}
\usepackage{subfigure}
\usepackage{amsmath}
\usepackage{eqparbox}
\usepackage{mathtools}
\usepackage{algorithm2e}
\RestyleAlgo{ruled}
\usepackage[export]{adjustbox}
\usepackage{cleveref}

\usepackage[colorinlistoftodos]{todonotes}

\newcommand{\inner}[3][\epsilon]{\ensuremath{\langle #2 , #3 \rangle_{#1}}}

\DeclareMathOperator*{\argmin}{\mathrm{argmin}}
\DeclareMathOperator*{\minimize}{\mathrm{minimize}}

\newtheorem{assumption}{Assumption}

\newcommand{\revzero}[1]{{\color{black}#1}}

\usepackage{lastpage}

\ShortHeadings{Distributed MCMC Sampling via ADMM}{Tzikas et al.}
\firstpageno{1}

\begin{document}

\title{Distributed Markov Chain Monte Carlo Sampling based on the Alternating Direction Method of Multipliers}

\author{\name Alexandros E. Tzikas \email alextzik@stanford.edu \\
       \addr Department of Aeronautics and Astronautics\\
       Stanford University\\
       Stanford, CA 94305, USA
       \AND
       \name Licio Romao \email licio@stanford.edu \\
       \addr Department of Aeronautics and Astronautics\\
       Stanford University\\
       Stanford, CA 94305, USA
       \AND
       \name Mert Pilanci \email pilanci@stanford.edu \\
       \addr Department of Electrical Engineering\\
       Stanford University\\
       Stanford, CA 94305, USA
       \AND
       \name Alessandro Abate \email aabate@cs.ox.ac.uk \\
       \addr Department of Computer Science\\
       University of Oxford\\
       Oxford,  OX1 3QD, United Kingdom
       \AND
       \name Mykel J. Kochenderfer \email mykel@stanford.edu \\
       \addr Department of Aeronautics and Astronautics\\
       Stanford University\\
       Stanford, CA 94305, USA
       }

\maketitle

\begin{abstract}%
\revzero{Many machine learning applications require operating on a spatially distributed dataset. Despite technological advances,
privacy considerations and communication constraints may prevent gathering the entire dataset in a central unit. In this paper, we propose a distributed sampling scheme based on the alternating direction method of multipliers, which is commonly used in the optimization literature due to its fast convergence. In contrast to distributed optimization, distributed sampling allows for uncertainty quantification in Bayesian inference tasks. We provide both theoretical guarantees of our algorithm's convergence and experimental evidence of its superiority to the state-of-the-art. For our theoretical results, we use convex optimization tools to  establish a fundamental inequality on the generated local sample iterates. This inequality enables us to show convergence of the distribution associated with these iterates to the underlying target distribution in Wasserstein distance. In simulation, we deploy our algorithm on linear and logistic regression tasks and illustrate its fast convergence compared to existing gradient-based methods.}
\end{abstract}

\begin{keywords}
  Markov chain Monte Carlo, distributed algorithms, sampling, alternating direction method of multipliers, proximal operator
\end{keywords}

\pagebreak

\section{Introduction}

\revzero{Since the 1950s, with the foundational works by \citet{dantzig1963linear} and later developments in the 1980s \citep{bertsekas2015parallel}, various research communities have recognized the importance of distributing computation to improve scalability. For example, the robotics community has explored multi-robot simultaneous localization and planning, multi-robot target tracking, and multi-robot task assignment \citep{shorinwa2023distributed1}.}
\revzero{The machine learning community has explored federated learning \citep{li2019convergence} and distributed training of neural networks \citep{yu2022dinno}.}

In order to select parameters in statistical models,  \revzero{optimization methods can be used to generate point estimates that aim at maximizing a given performance metric.} 
A different approach for selecting parameters relies on a Bayesian treatment, whose goal is to obtain samples from the posterior distribution on the parameter space \citep{andrieu2003introduction, sekkat2022large}.
\revzero{Sampling methods constitute an important module in such a Bayesian paradigm, because they allow us to retain a full probabilistic framework of the uncertainty affecting our statistical model. 
Such Bayesian approaches are usually employed to avoid overfitting \citep{bhar2023distributed,andrieu2003introduction} and perform uncertainty quantification \citep{bhar2023distributed,andrieu2003introduction}. Techniques based on Markov chain Monte Carlo (MCMC) \citep{andrieu2003introduction}, variational inference \citep{blei2017variational}, and Langevin dynamics \citep{welling2011bayesian} have been proposed to perform sampling.}

\revzero{In this paper, we seek a distributed sampling mechanism for two reasons. First, there are applications in which the available data is spatially distributed, and collecting such data in a central processing unit is not feasible due to privacy issues, communication constraints or simply the size of the data. Second, the agents that hold the local, private data are usually equipped with computational power, thus enabling local computations using a (small) subset of the entire dataset.}

\revzero{
Our proposed sampling scheme leverages the consensus alternating direction method of multipliers (ADMM), termed C-ADMM, presented by \citet{mateos2010distributed, shorinwa2023distributed1, shorinwa2023distributed2}, and \citet{shi2014linear}. A unique and crucial feature of the proposed sampling scheme, which we refer to as the distributed ADMM-based sampler (D-ADMMS), is the addition of a noise term in the proximal step of C-ADMM. As opposed to existing literature \citep{gurbuzbalaban2021decentralized} that relies on gradient computations for the parameter updates, the added noise and proximal updates of D-ADMMS are essential to its superior convergence behavior. In summary, our main contributions are as follows:
\begin{itemize}
    \item We show how to adapt C-ADMM to perform sampling in a distributed manner, by designing a new distributed sampling algorithm, which is termed D-ADMMS.
    \item We develop a new analysis to show convergence of the distribution of the generated iterates of D-ADMMS to the target distribution. 
    \item We study the performance of the proposed scheme on regression tasks and discuss advantages with respect to standard Langevin dynamics.
\end{itemize}
}

\revzero{
The remainder of the paper is structured as follows. In section $2$, we review related work. In section $3$, we formulate the problem, while in section $4$ we describe our proposed approach. In section $5$, we detail the convergence analysis of our proposed algorithm. Section $6$ contains the numerical experiments. Finally, we conclude with closing remarks and a discussion of future work in section $7$.}

\section{Related Work}

\textbf{Langevin and Hamiltonian Gradient MCMC.}
When the goal is to sample from a target distribution in Gibbs form ($\propto \exp{-U(x)}$), discretizations 
of stochastic differential equations (SDEs) \revzero{are an attractive sampling technique because the stationary distribution of these discretizations is usually close to the target distribution.} Langevin algorithms are MCMC methods based on the discretization of the overdamped Langevin diffusion, while Hamiltonian algorithms are MCMC methods based on the underdamped Langevin diffusion. Both Langevin and Hamiltonian methods scale well with high-dimensional sampling \revzero{spaces} \citep{kungurtsev2023decentralized}. They use first-order information of the target distribution to guide the dynamics towards the relevant regions of the parameter space. The stationary distribution of such algorithms, for example the unadjusted Langevin algorithm (ULA) \citep{parayil2021decentralized}, may contain a bias with respect to the stationary distribution of the underlying SDE. \revzero{To mitigate this bias, a Metropolis-Hastings correction can be introduced at the expense of increasing  computation.} Stochastic optimization tools were combined with Langevin dynamics by \cite{welling2011bayesian} to obtain a \textit{single-agent} MCMC algorithm that uses mini-batches of data at every iteration (to obtain \textit{gradient estimates}), along with a variable step-size and noise variance to guide convergence to the desired distribution.

\revzero{These methods} have recently been extended to the distributed setting, where the goal is to sample from a distribution, which is proportional to $\exp{-\sum_{i \in \mathcal{V}} f_i(x)}$, \revzero{in a network of a set $\mathcal{V}$ of agents}. The function $f_i$, for $i \in \mathcal{V}$, is only known to the corresponding agent $i$. Agents can communicate with their direct neighbors, as defined by the set of edges in the communication graph. \revzero{The methods that have been studied include: a modified version of \textit{distributed stochastic gradient descent} \citep{nedicdistributedopt}, namely distributed SGLD (D-SGLD) \citep{gurbuzbalaban2021decentralized}, and a method called \textit{distributed stochastic gradient Hamiltonian Monte Carlo} (D-SGHMC) \citep{gurbuzbalaban2021decentralized}, which is an adaptation of the SGHMC method to the distributed setting.} SGHMC can be faster \revzero{than} SGLD, because it is based on the discretization of the underdamped inertial Langevin diffusion, which converges to the stationary distribution faster than its overdamped counterpart due to a momentum-based accelerating step. \cite{gurbuzbalaban2021decentralized} provide convergence guarantees of the probability distribution of the local iterates of each agent to the target distribution in terms of Wasserstein distance. 

A \textit{distributed Hamiltonian Monte-Carlo} algorithm with a Metropolis acceptance step was recently derived \citep{kungurtsev2023decentralized}. Each agent estimates \textit{both first-order and second-order information} of the global potential function of the target distribution, \revzero{making this approach} different \revzero{from} ours. The ULA, \textit{which is gradient-based}, has also been modified for the \textit{distributed} case, giving rise to the D-ULA scheme \citep{parayil2021decentralized}. 
Assuming conditional independence among the data, the posterior is given as a product of local posteriors \revzero{that} are used to define the local dynamics. \textit{The distributed gradient-based ULA} \citep{parayil2021decentralized} has been modified to reduce the communication requirements between the agents by \cite{bhar2023distributed}. The difference is that the local iterate is not shared at every iteration but \revzero{only asynchronously by a triggering mechanism, which is based on the iterate's variation.}

It should be evident from this discussion that \textit{\revzero{m}ost distributed sampling algorithms are gradient-based.} \revzero{The optimization literature suggests that stochastic first-order methods are usually not the fastest to converge \citep{ryu2022large}, and} can be sensitive to the choice of hyperparameters \citep{toulis2021proximal}. At the same time, C-ADMM offers fast convergence in distributed optimization tasks \citep{shorinwa2023distributed2}. Motivated by this, we focus on developing an ADMM-inspired distributed sampling scheme.

\textbf{Distributed MCMC.}
In the case of large datasets, data is divided among agents. Sampling from the global posterior in this case can be done \revzero{using parallelized} MCMC. \revzero{\citet{neiswanger2014asymptotically} develop an MCMC sampling framework where each processing unit contains part of the dataset.} Each agent \revzero{deploys an MCMC method} independently (without communicating) \revzero{to sample from} the product term of the global posterior related to its dataset. \revzero{Then,} by appropriately combining these individual samples, samples from the global posterior can be obtained, in the spirit of a divide-and-conquer scheme. \textit{The combination procedure nevertheless requires a central coordinator.} An \revzero{alternative approach is to use several Markov chains.} 
\revzero{\citet{distributedMCMC_Welling} propose such a \textit{distributed sampling algorithm based on stochastic gradient Langevin dynamics (SGLD)}. However, \textit{it requires a central coordinator}.}

\textbf{MCMC and Federated Learning.}
\textit{The paradigm of SGLD has been used in the field of federated learning \citep{deng2022convergence} \revzero{to convert optimization algorithms into }sampling algorithms. }\revzero{In federated learning,} multiple \revzero{agents} aim to jointly optimize an objective without sharing their private data. The \revzero{agents} containing the private data can communicate with a centralized unit. 
\revzero{E}very iteration consists of a broadcast step, multiple local gradient steps by each \revzero{agent} using its local function, and finally a consensus step. Inspired by SGLD, the federated averaging optimization algorithm was modified into a \textit{gradient-based sampling algorithm by adding noise in the local gradient updates \citep{deng2022convergence}.} 
A combination of \textit{gradient-based Langevin dynamics} and compression techniques to reduce the communication cost has been studied \citep{karagulyan2023elf}. \textit{Nevertheless, the algorithm is not applicable in a distributed setting because it assumes a centralized processing unit. } 

A generalization of the federated averaging algorithm to compute the mode of the posterior distribution has been explored \citep{alshedivat2021federated}. \textit{The centralized processing unit performs gradient steps on a suitable objective, while the \revzero{agents} employ a variant of stochastic gradient MCMC in order to compute local covariances and expectations.} Finally, \textit{a distributed method} that minimizes the Kullback-Leibler (KL) divergence with the data likelihood function extends federated learning \citep{lalitha2019decentralized}. It consists of a local Bayesian update, \textit{a projection onto the allowed family of posteriors}, and a consensus step.

\textbf{Proximal Langevin Algorithms.} ADMM \revzero{has} a proximal update \revzero{at} each iteration. 
\revzero{T}he literature \revzero{suggests} that proximal operators are more stable than subgradients \citep{bauschke10convex}. 
\revzero{P}roximal optimization algorithms can be viewed as discretizations of gradient flow differential equations, whose equilibria are the minimizers of the considered function \citep{parikh2014proximal}.
Analogously, although common forms of Langevin MCMC methods employ subgradients, proximal MCMC methods possess favorable convergence and efficiency properties \citep{durmus2018efficient, salim2019stochastic, pereyra2016proximal}. While the classical ULA is based on a forward Euler approximation of the Langevin SDE that has the target distribution as the stationary distribution, proximal Langevin algorithms are based on discretizing an SDE whose stationary distribution equals the target distribution's Moreau approximation \citep{pereyra2016proximal}. The regularity properties of the Moreau approximation function lead to discrete approximations of the SDE with favourable stability and convergence qualities. To correct for the incurred error, a Metroplis-hastings accept-reject step has been proposed  \citep{pereyra2016proximal}. However, \textit{this algorithm requires knowing the full energy function of the desired distribution at every iteration and therefore it is not suitable for distributed computation.}

A proximal stochastic Langevin algorithm has been proposed  to sample from a distribution $\propto \exp{-U(x)}$ with $U(x) = F(x) + \sum_{i} G_i(x)$, where $F$ is smooth convex and $G_i$ are (possibly non-smooth) convex functions \citep{salim2019stochastic}. The authors assume that $F(x)$ and $G_i(x)$ can be written as expectations of functions $f(x, \xi)$ and $g_i(x, \xi)$, where $\xi$ \revzero{is} a random variable. This allows the use of stochastic information on $F(x)$ \revzero{and} $G_i(x)$ when designing the algorithm. At every iteration, a stochastic \textit{gradient} step is taken with respect to $F(x)$ and Gaussian noise is added. Then, stochastic proximal operators of each $G_i$ are deployed sequentially. \textit{Although our problem fits in this problem formulation, the algorithm by \cite{salim2019stochastic} is not suitable for deployment in a distributed network setting, \revzero{as} it requires the sequential deployment of the proximal operators. }

\textbf{Connections between ADMM, Sampling, and SDEs.}
ADMM is an optimization method that has been developed to combine the robustness and convergence of the method of multipliers and the decomposability of dual ascent \citep{boyd2011distributed}. 
\citet{vono2019split} propose fast and efficient variations of the Gibbs sampler, based on the idea of variable splitting and variable augmentation, in order to sample from $\exp{-\sum_i f_i(x)}$. The methods employ surrogate distributions, which converge in the limit to the target distribution, and consist of sequentially sampling a variable from its conditional distribution on the remaining ones. If\revzero{,} instead of sampling, we perform MAP optimization at each step of their proposed algorithm, we recover the ADMM optimization method. \textit{Although this method is closely related to our proposed scheme, it requires centralized communication at every step, as \revzero{discussed in} \citet[appendix B]{vono2019split}}. Other distributed sampling methods have been proposed that are inspired by ADMM and the variable-splitting idea in optimization. \cite{rendell2020global} introduce auxiliary variables and construct an MCMC algorithm on an extended state space. \revzero{Their algorithm} can be partly deployed in a distributed manner among agents, because the auxiliary variables are conditionally independent given the variable of interest. The auxiliary variables can be independently sampled at each agent given the variable of interest. \textit{The algorithm requires a centralized machine to sample the variable of interest because the method consists of \revzero{alternating between} sampling the agent auxiliary variables given the variable of interest and vice versa.} An instance of this approach, the split Gibbs sampler, has been studied by \citet{vono2022efficient}.

Connections between ADMM and SDEs are made for the case of a stochastic optimization problem where the sum of a function $g(x)$ and the expected value of another function $f(x, \xi)$, where $\xi$ is a random variable, is to be minimized \citep{zhou2020stochastic}. At every iteration, the primal variables are updated using a random sample of $f$, namely $f(x, \xi_i)$. \cite{zhou2020stochastic} prove that the primal iterate converges as the step size decreases, \revzero{in some sense}, to the stochastic process satisfying a given SDE. \textit{Different from our formulation, the only noise in the setting of \citet{zhou2020stochastic} stems from the noisy sample of $f$at every iteration.}

\section{Problem Formulation}\label{sec:problem_form}

We consider a network of agents. Each agent $i$ possesses a local function $f_i(x)$, where \revzero{$x \in \mathbb{R}^d$}. In machine learning applications, $f_i(x)$ could pertain to the loss, as a function of the model's parameter $x$, on the agent $i$'s dataset. As such, $f_i(x)$ is considered unknown to all agents other than agent $i$. This is due to privacy considerations and the high cost of transmitting agent $i$'s data across the network. The overall goal is to sample in a distributed manner from the distribution
\begin{equation} \label{eq: target_distribution}
    \mu^*(x) \propto \exp{-F(x)},\ F(x) = \sum_{i \in \mathcal{V}} f_i(x).
\end{equation}
Such a log-concave function arises as the posterior distribution in various Bayesian inference problems, such as distributed Bayesian linear and logistic regression \citep{gurbuzbalaban2021decentralized}.

The communication topology of the network is characterized by an undirected graph  $\mathcal{G} = \left( \mathcal{V}, \mathcal{E} \right)$, where  $\mathcal{V} = \lbrace 1, \dots, N \rbrace$, for some integer $N$, is the set of agents, and $\mathcal{E} \subset \mathcal{V} \times \mathcal{V}$ is the set of communication links, i.e., $(i, j) \in \mathcal{E}$ if and only if $i \neq j$ and node $i$ can communicate directly with node $j$. 
The neighborhood of agent $i$ is denoted $\mathcal{N}_i = \lbrace j \mid (i, j) \in \mathcal{E}\rbrace$. The cardinality of $\mathcal{N}_i$ is denoted $N_i$. 
Complementary to the undirected graph $\mathcal{G}$, we also describe the existing communication topology via a directed graph, $\mathcal{G}_d = \left( \mathcal{V}, \mathcal{A} \right)$. Every edge $e \in \mathcal{E}$ is associated with two directed links in $\mathcal{A}$ that connect the same nodes as $e$. Therefore the cardinality of $\mathcal{A}$ is twice that of $\mathcal{E}$, i.e., $\lvert \mathcal{A}\rvert = 2\lvert \mathcal{E} \rvert$, and $\mathcal{G}_d$ describes the same topology as $\mathcal{G}$.

\section{Description of the Proposed Method}
In this section, we introduce our proposed method, D-ADMMS. We start by providing background material on distributed optimization and C-ADMM. We then describe how we modify C-ADMM, which is used for distributed optimization, in order to obtain D-ADMMS, which performs distributed sampling.

\subsection{Background on C-ADMM for Distributed Optimization}

In distributed optimization problems, we consider the set-up introduced in section \ref{sec:problem_form}. However, in distributed optimization we aim to solve the optimization problem
\begin{equation}\label{eq:distr_opt}
    \minimize_{x \in \mathbb{R}^d} \quad \sum_{i \in \mathcal{V}} f_i(x),
\end{equation}
instead of sample from eq. \eqref{eq: target_distribution}.
We may introduce a local optimization variable, $x_i$ for each agent $i$, and consensus constraints in order to obtain the optimization problem
\begin{equation}\label{eq:distr_opt_2}
\begin{aligned}
    \minimize_{\substack{\lbrace x_i \rbrace_{i \in \mathcal{V}},\\ \lbrace z_{i,j} \rbrace_{(i,j) \in \mathcal{E}}}} \quad  & \sum_{i \in \mathcal{V}} f_i(x_i) \\
    \mathrm{subject\ to} \quad  &x_i =x_j\ \forall (i,j) \in \mathcal{E}.\\
\end{aligned}
\end{equation}
If $\mathcal{G}$ is connected, $x^*$ is an optimal point of problem \eqref{eq:distr_opt} if and only if $x_i = x^*,\ \forall i \in \mathcal{V},$ is an optimal point of problem \eqref{eq:distr_opt_2}. Problem \eqref{eq:distr_opt_2} lends itself to a distributed treatment.

C-ADMM is a distributed optimization algorithm inspired by the method of multipliers, which computes a primal-dual solution pair for the optimization problem via the augmented Lagrangian: the primal variables $x_i$ are updated as the minimizers of the augmented Lagrangian and the dual variables are updated via (dual) gradient ascent on the augmented Lagrangian \cite{shorinwa2023distributed2}. C-ADMM introduces auxiliary optimization variables to problem \eqref{eq:distr_opt_2} for each consensus constraint, which allows for distributed update steps. At step $(k+1)$, the primal variable of agent $i$, $x_i^{(k+1)}$, is updated according to
\begin{equation}
    x_i^{(k+1)} \gets \argmin_{x} \left\{ f_i(x) +
     p_i^{{(k)}^T} x+ \rho \sum_{j \in \mathcal{N}_i} \left\| x - \frac{x_i^{(k)}+x_j^{(k)}}{2} \right\|_2^2 \right\},
\end{equation}
while the dual variable of agent $i$ at step $(k+1)$, $p_i^{(k+1)}$, which corresponds to the consensus constraints involving agent $i$ and its neighbors, is updated according to
\begin{equation}
    p_i^{(k+1)} \gets p_i^{(k)} + \rho \sum_{j \in \mathcal{N}_i} \left(x_i^{(k+1)} - x_j^{(k+1)}\right),
\end{equation}
with initialization at zero.

\subsection{The Proposed Method: D-ADMMS}

Our distributed MCMC algorithm, D-ADMMS, is a modified version of the C-ADMM optimization algorithm \citep{shorinwa2023distributed2}. In constrast to C-ADMM, which is a distributed optimization algorithm, D-ADMMS is a distributed sampling algorithm. D-ADMMS is given in Algorithm \ref{alg:proposed}. \revzero{A key feature of the proposed sampling scheme} is the added noise in the update of the primal variables. The MCMC sample corresponds to the local primal variable $x_i^{(k)}$. At every iteration, each agent updates its local primal iterate by solving a proximal problem. The primal update of D-ADMMS 
can equivalently be written as
\begin{equation}
    x_i^{(k+1)} = \mathrm{prox}_{\gamma_i f_i} \Bigg \{ \sum_{j \in \mathcal{N}_i} \frac{x_i^{(k)}+x_j^{(k)}}{2N_i} + \dfrac{\sqrt{2}}{2\rho}w_i^{(k+1)} + \dfrac{p_i^{(k)}}{2\rho N_i} \bigg \},
\end{equation}
where $\gamma_i = {2}/\left({\rho N_i}\right)$. 
Each agent then communicates its primal variable to its neighbors and updates its dual variable $p_i^{(k)}$ based on the disagreement of the primal variables of the neighboring agents. 
\revzero{The inspiration of the added noise in the proximal step is derived from the algorithm by \citet{salim2019stochastic}.}
The first step at each iteration of the algorithm by \citet{salim2019stochastic} consists of a noiseless gradient step and then a proximal update with added noise. In D-ADMMS, the noiseless gradient step corresponds to the update of the dual variables, while the primal variables are updated with a noisy proximal step. The scaling of the noise involved is however different between the two algorithms. 

\SetKw{Init}{Initialization:}{}{}
\SetKw{Output}{Output:}{}{}
\SetKw{Pars}{Parameters:}{}{}
\SetKwProg{l}{do in parallel}{}{}
\begin{algorithm}
\caption{Proposed Algorithm (D-ADMMS)}\label{alg:proposed}
\Init{$k \leftarrow 0,\ x_i^{(k)} \in \mathbb{R}^d, p_i^{(k)} = \mathbf{0}\ \forall i \in \mathcal{V}$}

\Pars{$\rho >0$}

\Output{samples $x_i^{(k+1)}\ \forall i \in \mathcal{V}$}

\l{$\forall i \in \mathcal{V}$}{

$w_i^{(k+1)} \sim \mathcal{N}(0, I)$

$x_i^{(k+1)} \gets \argmin_{x} \left\{ f_i(x) +
     p_i^{{(k)}^T} x+ \rho \sum_{j \in \mathcal{N}_i} \left\| x - \frac{x_i^{(k)}+x_j^{(k)}}{2} + \dfrac{\sqrt{2}}{2\rho}w_i^{(k+1)}\right\|_2^2 \right\}$

Communicate $x_i^{(k+1)}$ to neighbors $j \in \mathcal{N}_i$

Receive $x_j^{(k+1)}$ from neighbors $j \in \mathcal{N}_i$

$
    p_i^{(k+1)} \gets p_i^{(k)} + \rho \sum_{j \in \mathcal{N}_i} \left(x_i^{(k+1)} - x_j^{(k+1)}\right)
$

$k \gets k+1$
}
\end{algorithm}

\section{Theoretical Analysis of the Proposed Method}

We study the convergence of the distribution associated with the primal iterates, $x_i^{(k)}$, of the proposed algorithm, D-ADMMS, to the target distribution $\mu^*(x)$, in terms of $2$-Wasserstein distance. The $2$-Wasserstein distance between two probability measures $\mu$ and $\nu$ with finite second moments is defined as 
\begin{equation}
    W(\mu, \nu) = \left( \inf_{\tau \in \Gamma(\mu, \nu)} \mathbb{E}_{(x,y)\sim \tau}\ \lVert x-y \rVert^2 \right)^{1/2},
\end{equation}
where $\Gamma(\mu, \nu)$ is the set of all couplings between $\mu$ and $\nu$ \revzero{\citep{villani2009wasserstein}}.
We adopt a convex analysis perspective. We base our analysis of the convergence rate (with respect to the iterates' Wasserstein distance to the target distribution) of D-ADMMS on the analysis of the convergence rate (with respect to the iterates' Euclidian distance to the optimal point) of C-ADMM by \citet{shi2014linear}. Modifying the analysis by \citet{shi2014linear} for our purposes is not trivial for two reasons: i) the added noise in the primal update of D-ADMMS gives rise to terms that do not exist in C-ADMM, and ii) it is not straight forward how to obtain a Wasserstein distance bound on the distribution of the iterates from a Euclidian distance bound of the iterates. We use $\lVert \cdot \rVert$ for the standard Euclidean norm and $\lVert \cdot \rVert_G^2$ for the $G$-matrix norm $(\cdot)^T G (\cdot)$.

This section is organized as follows. We start the first subsection by stating our assumptions and introducing helpful quantities. We finish it with a lemma that shows the equivalence of the D-ADMMS updates and a different set of updates, which involve the same primal variables. This second set of updates is used in our analysis. In the second subsection, we prove a recursive inequality for the Wasserstein distance between the distribution of the primal iterates of D-ADMMS and the target distribution of eq. \eqref{eq: target_distribution}. Our main result is Theorem 3, which is included in the third subsection. Theorem 3 states that there exists a decreasing upper-bound on the Wasserstein distance between the distribution of the primal iterates of D-ADMMS and the target distribution of eq. \eqref{eq: target_distribution}, as the iterations evolve.

\subsection{Assumptions, Definitions, and an Equivalent Expression of D-ADMMS}
\begin{assumption}
The local objective functions $f_i(x)$ are strongly convex: $\forall x_a, x_b \in \mathbb{R}^d, \forall i \in \mathcal{V}$, it holds that
\begin{equation*}
    \langle \nabla f_i(x_a) - \nabla f_i(x_b), x_a -x_b \rangle \geq m_{f_i} \lVert x_a -x_b\rVert^2,\ m_{f_i} >0.
\end{equation*}
\end{assumption}
\begin{assumption}\label{ass:Lipsch}
The gradients of the local objective functions are Lipschitz continuous: $\forall x_a, x_b \in \mathbb{R}^d, \forall i \in \mathcal{V}$, it holds that
\begin{equation*}
    \lVert \nabla f_i(x_a) - \nabla f_i(x_b)\rVert \leq M_{f_i} \lVert x_a -x_b\rVert,\ M_{f_i} >0.
\end{equation*}
\end{assumption}
\begin{assumption}
    The graph topology $\mathcal{G}$ is connected.
\end{assumption}

We further define the consensus convex optimization problem associated with $\mu^*$ as
\begin{equation}\label{eq:opt_problem}
\begin{aligned}
    \minimize_{\substack{\lbrace x_i \rbrace_{i \in \mathcal{V}},\\ \lbrace z_{i,j} \rbrace_{(i,j) \in \mathcal{A}}}} \quad  & \sum_{i \in \mathcal{V}} f_i(x_i) \\
    \mathrm{subject\ to} \quad  &x_i = z_{i,j},\ x_j = z_{i,j}\ \forall (i,j) \in \mathcal{A}.\\
\end{aligned}
\end{equation}
Concatenating each $x_i$ in $X \in \mathbb{R}^{Nd}$ and all $z_{i,j}$ in $Z \in \mathbb{R}^{\lvert \mathcal{A} \rvert d}$, we may write the constraint of eq. (\ref{eq:opt_problem}) as
\begin{equation}
    AX+BZ = 0.
\end{equation}
Here, $A = \left[A_1; A_2 \right]$, where $A_1, A_2 \in \mathbb{R}^{\lvert \mathcal{A} \rvert d \times N d}$. If $(i,j) \in \mathcal{A}$ and $z_{i,j}$ is the $q$-th block of $Z$, then the $(q,i)$ block of $A_1$ and the $(q,j)$ block of $A_2$ are the the $d \times d$ identity matrices. All other blocks of $A_1, A_2$ contain zero entries. Also $B = \left[ -I_{\lvert \mathcal{A}\rvert d}; -I_{\lvert \mathcal{A}\rvert d}\right]$, where $I_{\lvert \mathcal{A}\rvert d}$ is the $\lvert \mathcal{A}\rvert d \times \lvert \mathcal{A}\rvert d$ identity matrix. We define $f(X) = \sum_{i \in \mathcal{V}} f_i(x_i)$. By Assumptions 1 and 2, $f(X)$ is strongly convex with constant $m_f = \min_i m_{f_i}$ and its gradients are Lipschitz continuous with constant $M_f = \max_i M_{f_i}$. 
\begin{assumption}
    $\sum_{i \in \mathcal{V}} f_i(x)$ admits a (unique) minimizer.
\end{assumption}

Problem (\ref{eq:opt_problem}) then admits a unique solution $(X^*, Z^*)$, because $\mathcal{G}$ is connected and $\sum_{i \in \mathcal{V}} f_i(x)$ admits a unique minimizer. It is evident that the optimization problem (\ref{eq:opt_problem}) is related to our problem of interest, which is to sample from distribution \eqref{eq: target_distribution} in a distributed manner, but solving problem (\ref{eq:opt_problem}) is not our objective. In our analysis, we use the minimizer of problem (\ref{eq:opt_problem}) as a fixed point to which the Euclidian distance of the primal iterates of D-ADMMS can be bounded. By assigning a point mass distribution to this fixed point, we are then able to obtain relations for the Wasserstein distance of the primal iterates of D-ADMMS to the target distribution.

We now introduce matrices $M_{-}, M_{+}, L_{-}, L_{+}$, and $D$, based on the network topology $\mathcal{G}$. $M_{+}$ and $M_{-}$ are the extended unoriented and oriented incidence matrices of $\mathcal{G}_d$, respectively. $L_{+}$ and $L_{-}$ are the extended signless and signed Laplacian matrices of $\mathcal{G}$, respectively. $D$ is the extended degree matrix of $\mathcal{G}$. By ``extended'' we mean the Kronecker product with $I_d$. We also denote $w^{(k)} = (w_1^{(k)}, \dots, w_N^{(k)}) \in \mathbb{R}^{Nd}$, and $p^{(k)} = (p_1^{(k)}, \dots, p_N^{(k)})\in \mathbb{R}^{Nd}$.

\begin{lemma}
Define $\beta \in \mathbb{R}^{\lvert \mathcal{A} \rvert d}$. The update equations of D-ADMMS in Algorithm \ref{alg:proposed} can be derived from the iterates
\begin{equation}\label{eq:noisy_primal}
    \nabla f(X^{(k+1)}) + M_{-}\beta^{(k+1)} +\sqrt{2} D w^{(k+1)} = \rho M_{+} \left(Z^{(k)} - Z^{(k+1)} \right),
\end{equation}
\vspace{-0.35cm}
\begin{equation}\label{eq:beta}
    \beta^{(k+1)} - \beta^{(k)} - \frac{\rho}{2} M_{-}^T X^{(k+1)} = 0,
\end{equation}
\vspace{-0.35cm}
\begin{equation}
    \frac{1}{2} M_{+}^T X^{(k)} - Z^{(k)} = 0, 
\end{equation}
where $X^{(k)}$ is the concatenation of the $x_i^{(k)}$ from Algorithm \ref{alg:proposed}.
\end{lemma} 
\begin{proof}
    The proof is included in Appendix \ref{sec:proof_lemma1}.
\end{proof}
By the lemma above, we may analyze the primal iterates of D-ADMMS using eq. (10-12).

The Karush-Kuhn-Tucker (KKT) conditions for problem (\ref{eq:opt_problem}) are
\begin{equation}
    \nabla f(X^*) + M_{-} \beta^* = 0,
\end{equation}
\vspace{-0.4cm}
\begin{equation}
    M_{-}^T X^* = 0,
\end{equation}
\vspace{-0.4cm}
\begin{equation}
    \frac{1}{2}M_{+}^TX^*-Z^* = 0,
\end{equation}
as described by \citet{shi2014linear}, where $\beta^*$ denotes the unique optimal multiplier that exists in the column space of $M_{-}^T$. Since the equality constraints of problem (\ref{eq:opt_problem}) are feasible, by Slater's condition \citep{boyd2004convex}, there exists an optimal multiplier $\Tilde{\beta}$ that satisfies the KKT conditions. Its projection onto the column space of $M_{-}^T$ is $\beta^*$, as analyzed by \citet{shi2014linear}. 

\begin{assumption}
    $\beta^{(0)}$ is in the column space of $M_{-}^T$.
\end{assumption}
\revzero{By inspection of eq. \eqref{eq:beta}, we observe that under Assumption 5, $\beta^{(k)}$ is in the column space of $M_{-}^T$ for all $k \geq 0$.}

\subsection{A Recursive Inequality of Convergence for D-ADMMS}
We define $U = (Z, \beta)$, $G = \mathrm{diag}\{\rho I_{2 \mid \mathcal{E} \mid d},\ \frac{1}{\rho} I_{2 \mid \mathcal{E} \mid d}\}$, and $U^* = (Z^*, \beta^*)$. We also denote the largest singular value and the smallest nonzero singular value of a matrix $M$, as $\sigma_{\mathrm{max}}(M)$, and $\sigma_{\mathrm{min}}(M)$ respectively.
We may now compute recursive bounds on $U^{(k)}$ and $X^{(k)}$. Then we can obtain recursive bounds on the Wasserstein distance of the iterate. We denote $\mu_{(\cdot)}$ as the probability distribution of random variable $(\cdot)$. Also $\boldsymbol{\mu}^*(X) = \prod_{i=1}^N \mu^*(x_i)$. Finally $W^2_G(\mu_{U^{(k)}}, \mu_{U^*}) = \rho W^2(\mu_{Z^{(k)}}, \mu_{Z^*}) + (1/\rho)W^2(\mu_{\beta^{(k)}}, \mu_{\beta^*}) $.

\begin{lemma}
    Under Assumptions (1-5), for any $\kappa > 1$, there exists a $\delta>0$, such that the distribution of the iterates of D-ADMMS satisfies the relation
    \begin{multline}
        W(\mu_{X^{(k+1)}}, \boldsymbol{\mu}^*) \leq \frac{1}{\sqrt{m_f}} W_G(\mu_{U^{(k)}}, \mu_{U^*}) + \\
        \frac{1}{\sqrt{2}m_f}\sqrt{\mathbb{E}\left( \lVert D w^{(k+1)} \rVert^2\right)} +
    W( \mu_{X^* - \frac{1}{\sqrt{2}m_f} D w^{(k+1)}},\boldsymbol{\mu}^*),
    \end{multline}
    and $W_G(\mu_{U^{(k)}}, \mu_{U^*})$ is recursively upper-bounded by
\begin{equation}\label{eq:lemma2_14_main}
    W_G(\mu_{U^{(k+1)}}, \mu_{U^*}) \leq \sqrt{a} W_G(\mu_{U^{(k)}}, \mu_{U^*}) +\frac{\mathbb{E}\left(y^{(k+1)}\right)}{2\sqrt{a}} + \sqrt{\mid\mathbb{E}\left( r^{(k+1)} \right) - \left( \frac{\mathbb{E}\left(y^{(k+1)}\right)}{2\sqrt{a}}\right)^2 \mid},
\end{equation}
where
\begin{equation}
   y^{(k+1)} = 2\frac{b}{\sqrt{m_f}} \Bar{w}^{(k+1)} +c \sigma_\mathrm{max}(M_{-}) \sqrt{\rho}\lVert D w^{(k+1)} \rVert + \frac{cd}{\sqrt{m_f}}\lVert D w^{(k+1)} \rVert,
\end{equation}
\begin{multline}
    r^{(k+1)} =  \frac{\sqrt{2}b}{m_f} \Bar{w}^{(k+1)} \lVert D w^{(k+1)} \rVert + b \left(\Bar{w}^{(k+1)}\right)^2 + \frac{b}{2m_f^2}\lVert D w^{(k+1)} \rVert^2 \\
    +\frac{\sqrt{2}cd}{m_f}\lVert D w^{(k+1)} \rVert^2  - e \lVert D w^{(k+1)} \rVert^2,
\end{multline}
\vspace{-0.3in}
\begin{multline}
    \Bar{w}^{(k+1)}=\lVert \frac{1}{\sqrt{2}m_f} D w^{(k+1)}\rVert + \lVert \sqrt{2} D w^{(k+1)} \rVert,\ \\
    \delta = \min \Biggl \{ \frac{(\kappa -1) \sigma^2_{\mathrm{min}}(M_{-})}{\kappa \sigma^2_{\mathrm{max}}(M_{+})}, \frac{m_f}{\frac{\rho}{4} \sigma^2_{\mathrm{max}}(M_{+}) + \frac{\kappa M_f^2}{\rho \sigma^2_{\mathrm{min}}(M_{-})} }\Biggr \}>0,
\end{multline}
and
\begin{multline}\label{eq:theorem_end}
    a =\frac{2m_f+1}{2m_f(1+\delta)},\ b=\frac{1}{2(1+\delta)},\ c=\frac{2\sqrt{2}\delta}{(1+\delta)\sigma^2_{\mathrm{min}}(M_{-})},\\
    d=\frac{\rho \sigma_\mathrm{max}^2(M_{-})}{2},\ e=\frac{2\delta}{(1+\delta)\rho \sigma^2_{\mathrm{min}}(M_{-})}.
\end{multline}
\end{lemma}
\begin{proof}
    The proof is included in Appendix \ref{sec:proof_lemma2}.
\end{proof}

\subsection{Our Main Result}
\revzero{By inspection of eq. (\ref{eq:lemma2_14_main}-\ref{eq:theorem_end}), if there exists a $\rho > 0$ such that $\delta$ satisfies $2 m_f \delta > 1$, then $a <1$ and we have convergence in terms of Wasserstein distance for the primal iterates $x_i^{(k)}$. This idea is formalized in the following theorem, whose result depends on the condition number of $f(X)$, $\tau_f = \frac{M_f}{m_f}$, and the condition number of the graph topology, $\tau_G = \frac{\sigma_\mathrm{max}(M_{+})}{\sigma_\mathrm{min}(M_{-})}$.
\begin{theorem}
    Assume Assumptions (1-5) hold and $\tau_f^{-1}\sqrt{\tau_f^{-2}+4\tau_G^{-2}}-\tau_f^{-2} > m_f^{-1}$ is true. Then, there exists a $\rho>0$ such that $a<1$ and 
    \begin{multline*}
        W(\mu_{X^{(k+1)}}, \boldsymbol{\mu}^*) \leq \frac{1}{\sqrt{m_f}} \left(\sqrt{a}\right)^{k} W_G(\mu_{U^{0}}, \mu_{U^*}) +
    \frac{1}{\sqrt{a m_f}}\frac{Y}{1-\sqrt{a}} +\\
    \frac{1}{\sqrt{m_f}}\frac{\sqrt{R}}{1-\sqrt{a}}+ 
        \frac{1}{\sqrt{2}m_f}\sqrt{\mathbb{E}\left( \lVert D w^{(k+1)} \rVert^2\right)} +
    W( \mu_{X^* - \frac{1}{\sqrt{2}m_f} D w^{(k+1)}},\boldsymbol{\mu}^*),
    \end{multline*}
    $Y$ and $R$ are upper bounds on the terms $\mathbb{E}\left(y^{(l)} \right)$ and $\mathbb{E}\left( r^{(l)} \right)$ respectively, and $Y, R \geq 0$ holds.
\end{theorem}}
\begin{proof}
    The proof is included in Appendix \ref{sec:proof_theorem1}.
\end{proof}

\revzero{Because of the definitions $L_{+}=\frac{1}{2}M_{+}M_{+}^T$ and $L_{-}=\frac{1}{2}M_{-}M_{-}^T$, we have that $\tau_G = \sqrt{\frac{\sigma_\mathrm{max}(L_{+})}{\sigma_\mathrm{min}(L_{-})}}$. $\sigma_\mathrm{min}(L_{-})$ is known as the graph's algebraic connectivity \citep{shi2014linear}, while $\sigma_\mathrm{max}(L_{+})$ is related to the node degrees of $\mathcal{G}$ \citep{cvetkovic2007signless}. Assuming a ring-cyclic graph topology of $5$ agents, we have $\tau_G = 1.7$. If we further assume $m_f = 2$, we get the sufficient condition for convergence: $\tau_f <1.23$. Increasing the connectedness of the graph, for a fully connected graph of $5$ agents, we have $\tau_G = 1.26$. Then, for $m_f = 2$, the sufficient condition becomes $\tau_f < \sqrt{6}$. We observe that as the graph becomes more connected, $\tau_G$ decreases. Then, fixing $\tau_f$ and $m_f$, Theorem 1 implies that D-ADMMS converges if the connectivity is above a certain threshold. In addition, $a$ decreases with increasing connectivity. The convergence upper bound however also contains terms, including $Y$ and $R$, which depend on noise and topology characteristics and can increase as the graph becomes more connected ($D$ has larger values in its diagonal and $a$ decreases). }

\revzero{We also note that the sufficient convergence condition $m_f^{-1}<\tau_f^{-1}\sqrt{\tau_f^{-2}+4\tau_G^{-2}}-\tau_f^{-2} $ is not symmetric with respect to the condition number $\tau_f$. In other words, scaling the functions $f_i$ by a common constant, which does not change $\tau_f$, modifies the convergence condition. This is because the scaling constant affects the solution of the proximal update for $x_i^{(k+1)}$}.

\revzero{Since ADMM is practically stable when the $f_i$ are not strongly convex, e.g., for indicator functions, our algorithm is also applicable in this case. However, the theoretical analysis above assumes strong convexity.}

\section{Simulation Results}

In this section, we test D-ADMMS in simulation. We first discuss the existing baselines from the literature and then provide simulation results for two different scenarios: Bayesian linear regression and Bayesian logistic regression.\footnote{The source code for the presented experiments can be found in the following repository: \url{https://github.com/sisl/Distributed_ADMM_Sampler}}

\subsection{Description of Baseline Algorithms}
We compare our proposed algorithm against D-SGLD \citep{gurbuzbalaban2021decentralized}, D-SGHMC \citep{gurbuzbalaban2021decentralized}, and D-ULA \citep{parayil2021decentralized}.
We let $S$ be a doubly stochastic matrix associated with the network's communication topology, i.e., a matrix whose $(i,j)$-entry, which we denote by $S_{ij}$, is non-negative and less than one, and is different from zero if and only if the entry $j \in \mathcal{N}_i$; whenever $j = i$, we let $S_{ii}$ be defined as $1 - \sum_{j \in \mathcal{N}_i} S_{ij}$.

D-SGLD is characterized by
\begin{equation}\label{eq:D-SGLD}
x_i^{(k+1)} = \sum_{j \in \mathcal{N}_i \cup \lbrace i \rbrace} S_{ij}x_j^{(k)} - \eta_{\mathrm{DSGLD}} \nabla f_i(x_i^{(k)})
+ \sqrt{2\eta_{\mathrm{DSGLD}}} w_i^{(k+1)},\ w_i^{(k+1)}\sim \mathcal{N}(0, I).
\end{equation} 
D-SGHMC proceeds with
\begin{equation}\label{eq: D-SGHMC_1}
v_i^{(k+1)} = v_i^{(k)} - \eta_{\mathrm{DSGHMC}} \left( \gamma v_i^{(k)} + \nabla f_i(x_i^{(k)})\right) +
\sqrt{2\gamma\eta_{\mathrm{DSGHMC}}}w_i^{(k+1)},\ w_i^{(k+1)}\sim \mathcal{N}(0, I),
\end{equation}
\begin{equation}\label{eq: D-SGHMC_2}
x_i^{(k+1)} = \sum_{j \in \mathcal{N}_i \cup \lbrace i \rbrace} S_{ij}x_j^{(k)} + \eta_{\mathrm{DSGHMC}} v_i^{(k+1)}.
\end{equation}
Finally, D-ULA evolves through the update
\begin{multline}\label{eq:D-ULA}
    x_i^{(k+1)} = x_i^{(k)} - \zeta^{(k)} \sum_{j \in \mathcal{N}_i} \left(x_i^{(k)}-x_j^{(k)} \right) - \alpha^{(k)} N \nabla f_i(x_i^{(k)})\\
    + \sqrt{2\alpha^{(k)}} w_i^{(k+1)}, w_i^{(k+1)}\sim \mathcal{N}(0, N I).
\end{multline}

\subsection{Bayesian Linear Regression}
\label{sec:bay_lin}
\subsubsection{Problem}
We study the distributed Bayesian linear regression setting. We consider a varying number of agents ($N=5,20,100$) and a varying network topology (fully connected, ring-cyclic, and no-edge) in our results. We assume that the model parameter is $x \in \mathbb{R}^2$, and we simulate IID data points $(z^l, y^l)$ for each agent through
\begin{equation}
    \delta^l \sim \mathcal{N}(0, \xi^2I),\ z^l \sim \mathcal{N}(0, I),\ y^l = x^Tz^l+\delta^l.
\end{equation}
We assign $\xi=4$. The prior on $x$ is $\mathcal{N}(0, \lambda I)$ with $\lambda=10$.
Each agent is assigned $n_i$ independent samples. The global posterior, from which we aim to sample, by a simple application of Bayes' rule, is of the form $\pi(x) \propto \exp{-\sum_{i \in \mathcal{V}} f_i(x)}$, where 
\begin{equation}
    f_i(x) = \dfrac{1}{2\xi^2}\sum_{l=1}^{n_i} \left(y_i^l - x^T z^l_i \right)^2 + \dfrac{1}{2\lambda N} \lVert x \rVert^2.
\end{equation}

\subsubsection{Algorithm}
We do not perform hyper-parameter tuning and assume $\eta_{\mathrm{DSGLD}}=0.009$, $\eta_{\mathrm{DSGHMC}}=0.1,\ \gamma=7$, as done by \citet{gurbuzbalaban2021decentralized}, while $\rho=5$. For D-ULA, we follow the logistic regression example by \citet{parayil2021decentralized} with modifications to avoid divergence. We assume $\alpha^{(k)} = {0.00082}/{(230+k)^{ \chi_2}}$, $\zeta^{(k)} = {0.48}/{(230+k)^{\chi_1}}$
We assume $\chi_1=\chi_2=0.05$ for the cyclic and no-edge topologies. For the case of the fully connected graph: when $N=5, 20$ we assume $\chi_1=0.55$ and $\chi_2=0.05$.
Note that the selection of parameters for D-ULA does not satisfy the conditions presented in the original paper \citep{parayil2021decentralized}, but the selected parameter values perform well experimentally. The algorithm by \cite{parayil2021decentralized} required the most testing to determine suitable parameters.
Note that $x_i^{(0)} \sim \mathcal{N}(0, I),\ \forall i \in \mathcal{V}$ for all algorithms. We also test two cases \revzero{for $n_i$ equal to $50$ and $200$}. \revzero{We tune the} parameters \revzero{independently} for both values of $n_i$. 

\subsubsection{Results}
We present results with respect to the $2$-Wasserstein distance between the empirical distribution of the iterates $x_i^{(k)}$ and the true posterior in Figures \ref{fig:bay_lin_reg_num_samples_[50]} and \ref{fig:bay_lin_reg_num_samples_[200]} for $n_i=50$ and $n_i=200$ respectively. The true posterior and the iterate distributions are Gaussian \citep{gurbuzbalaban2021decentralized}. Furthermore, the true posterior is known in closed form \citep{gurbuzbalaban2021decentralized}. The Gaussian distribution of each agent's iterate is estimated by empirically computing the expectation and covariance through $100$ independent trials. The same holds for the average agent iterate $\sum_{i \in \mathcal{V}} x_i^{(k)}/ N$. The closed form of the $2$-Wasserstein distance between Gaussians, in the case of non-singular covariances, only involves the covariances and means of the distributions \citep{givens1984class}. 

Figures \ref{fig:bay_lin_reg_num_samples_[50]} and \ref{fig:bay_lin_reg_num_samples_[200]} include the Wasserstein distance between the average iterate (avg) and the target distribution and the Wasserstein distance between the iterate of an agent (ag) and the target distribution. 
We observe that our proposed algorithm converges faster than the other presented schemes (D-SGLD, D-ULA, D-SGHMC) in the cases of sparsely connected network topology (ring-cyclic graph), while it can \revzero{become} slower for more connected graphs (fully connected). This agrees with results in the optimization literature, which show that C-ADMM can be slower as the communication topology becomes more connected \citep{bof2018admm}. In the case of a highly connected graph, the dual variables take a longer time to converge because of the larger number of inter-agent disagreement terms involved, which slows convergence. Although $a$ in Theorem 1 decreases with increasing connectedness, the theoretical upper bound need not decrease, because it includes terms that increase with increased connectivity. Therefore, the experimental behavior is not necessarily misaligned with the theoretical results.

We have also included the performance of ADMM (i.e., $w_i^{(k)}=0$ in Algorithm \ref{alg:proposed}) in Figures \ref{fig:bay_lin_reg_num_samples_[50]} and \ref{fig:bay_lin_reg_num_samples_[200]}. \revzero{ADMM is an optimization scheme and hence it drives $x_i^{(k)}$ to a point with an optimal value for the associated optimization problem, for any initialization of $x_i^{(k)}$}. 
We observe that in the no-edge topology, ADMM performs exactly the same as D-ADMMS. This is attributed to the last term of the primal update, which vanishes in the case of \revzero{a} no-edge topology \revzero{because} no agent has any neighbors. The superior performance (in Wasserstein distance) of ADMM compared to the sampling schemes for the case $N=100$ can be attributed to the very small ($\leq10^{-3}$) entries of the true posterior covariance in this case, because of the large number of samples present. This means that simply finding the MAP (which is also the mean because of Gaussian structure) with ADMM is enough for a small Wasserstein distance to the true posterior distribution. For $N=100$, the inability of CVXPY \citep{diamond2016cvxpy} to converge justifies the exclusion of the fully connected case.

The fast convergence (in cases) of D-ADMMS is intuitively  attributed to the proximal primal update of the algorithm. D-ADMMS is able to take large steps in terms of Wasserstein distance in the initial iterations because its primal update consists of completely solving an optimization problem. Even in the case of no-edge topology, we observe the large reduction of Wasserstein distance in the early iterations. Nevertheless, because each agent does not gather information from other agents in the update of its dual variables, the trajectory in the primal space remains uncorrected in this case. This leads D-ADMMS to converge farther from the target distribution.

\begin{figure}
     \begin{subfigure}
         \centering
        \includegraphics[width=0.8\columnwidth, center]{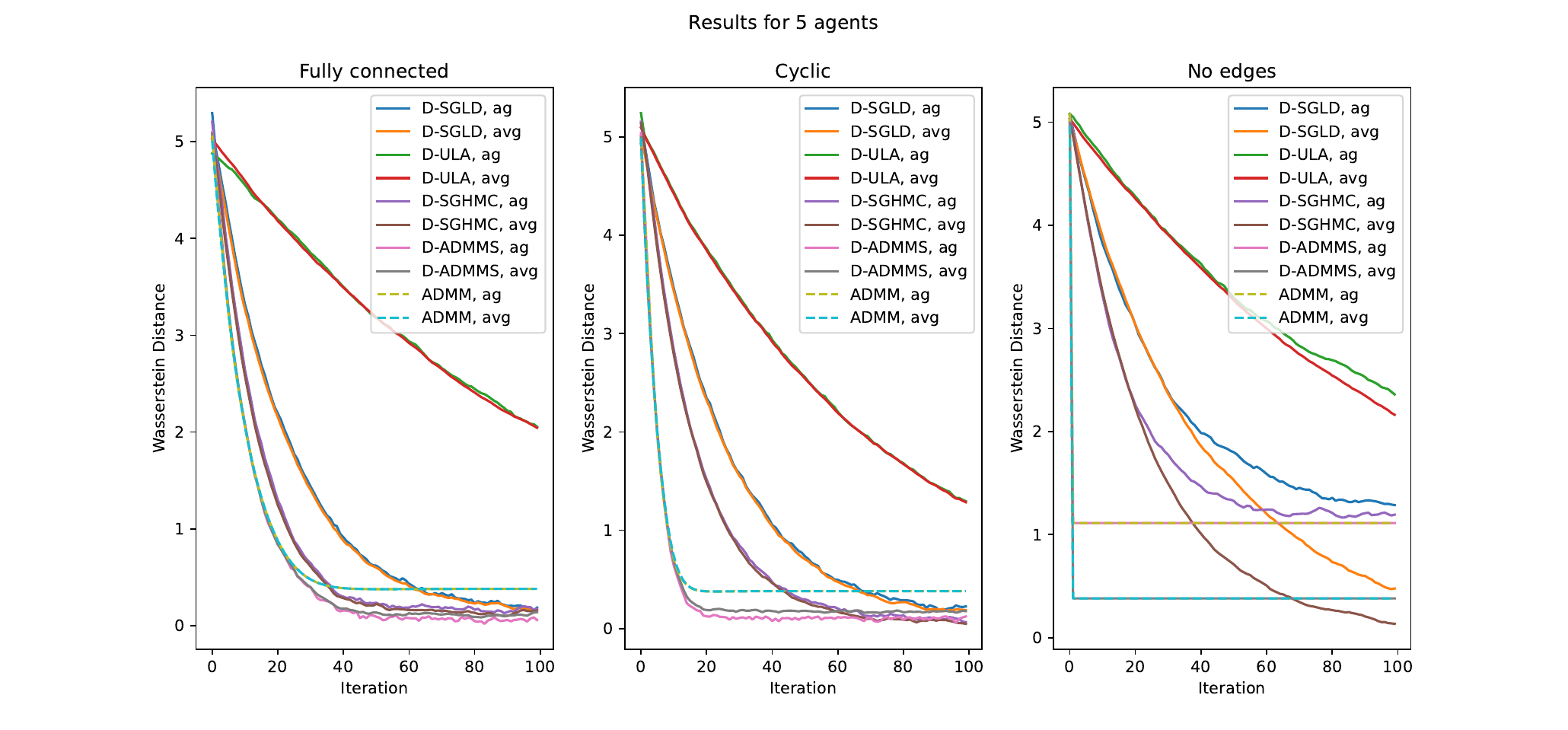}
         \label{fig:bay_lin_reg - num_agents [5] - num_samples [50]}
     \end{subfigure}
     \begin{subfigure}
         \centering
        \includegraphics[width=0.8\columnwidth, center]{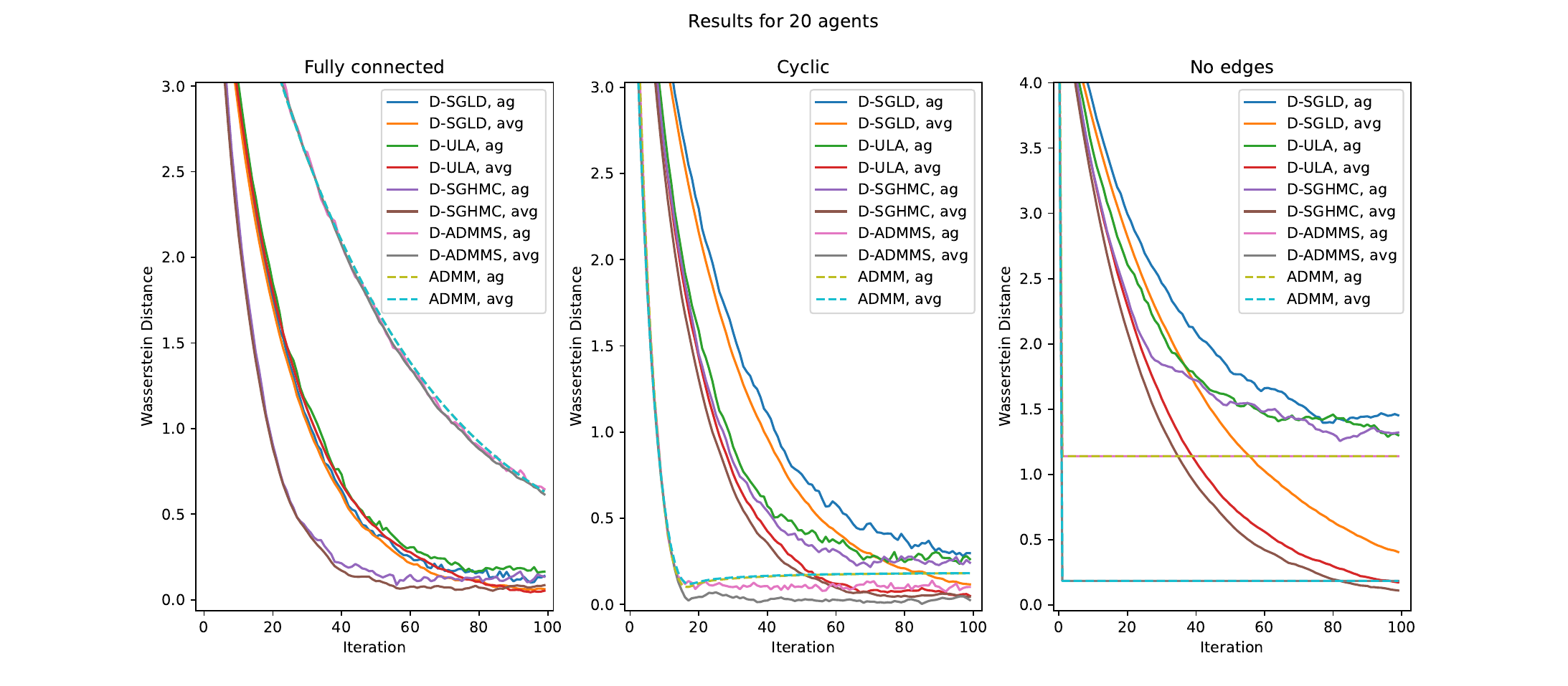}
         \label{fig:bay_lin_reg - num_agents [20] - num_samples [50]}
     \end{subfigure}
    \begin{subfigure}
         \centering
        \includegraphics[width=0.8\columnwidth, center]{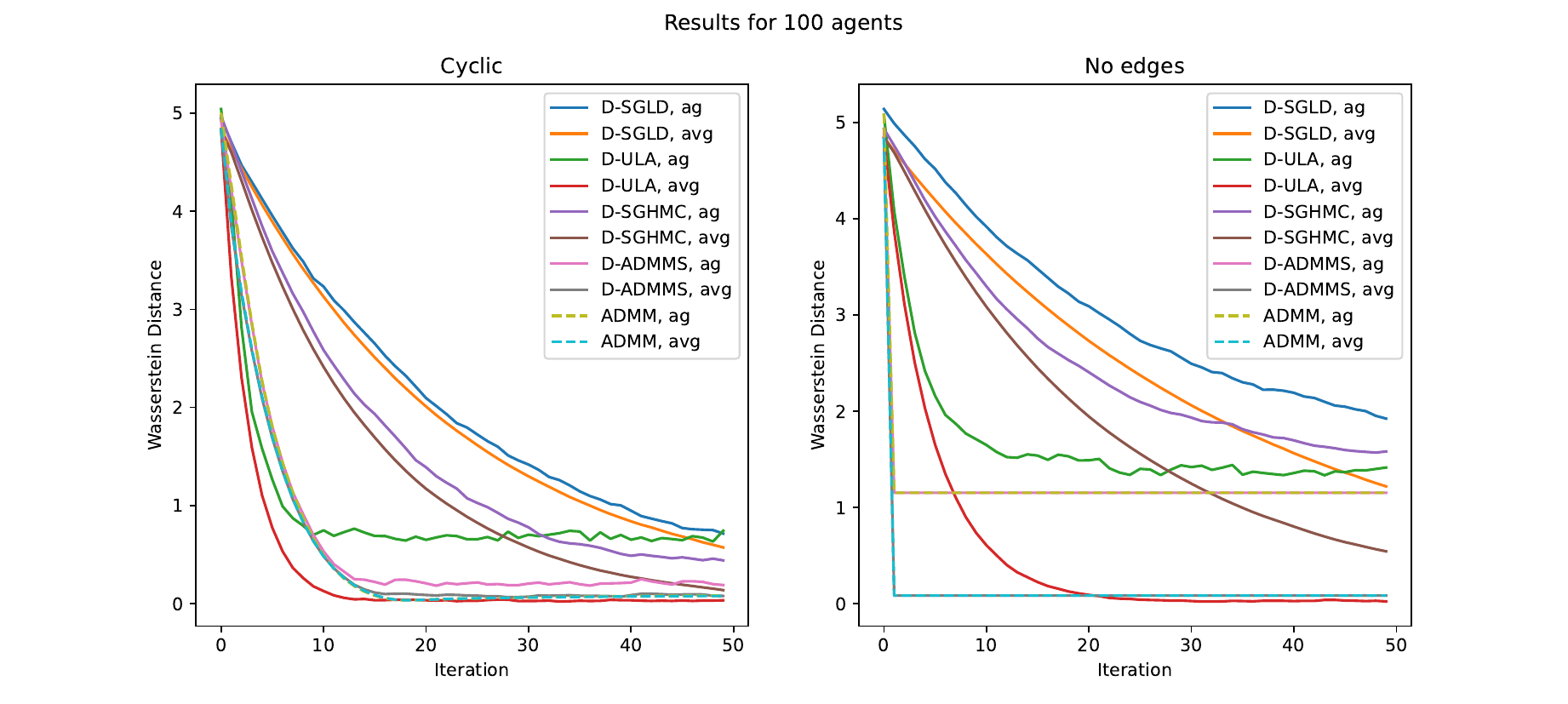}
         \label{fig:bay_lin_reg - num_agents [100] - num_samples [50]}
     \end{subfigure}
        \caption{$2$-Wasserstein distance to target distribution vs iteration for $n_i=50$. Both the distance to the target distribution of the average iterate (avg) and a specific agent iterate (ag) are provided for each method. For the sparsely connected (cyclic) graph topology, our proposed algorithm (D-ADMMS) outperforms the baselines (D-SGLD, D-ULA, D-SGHMC) in terms of Wasserstein distance between the distribution of the agent iterate and the target distribution.}
        \label{fig:bay_lin_reg_num_samples_[50]}
\end{figure}
\begin{figure}
     \begin{subfigure}
         \centering
        \includegraphics[width=0.8\columnwidth, center]{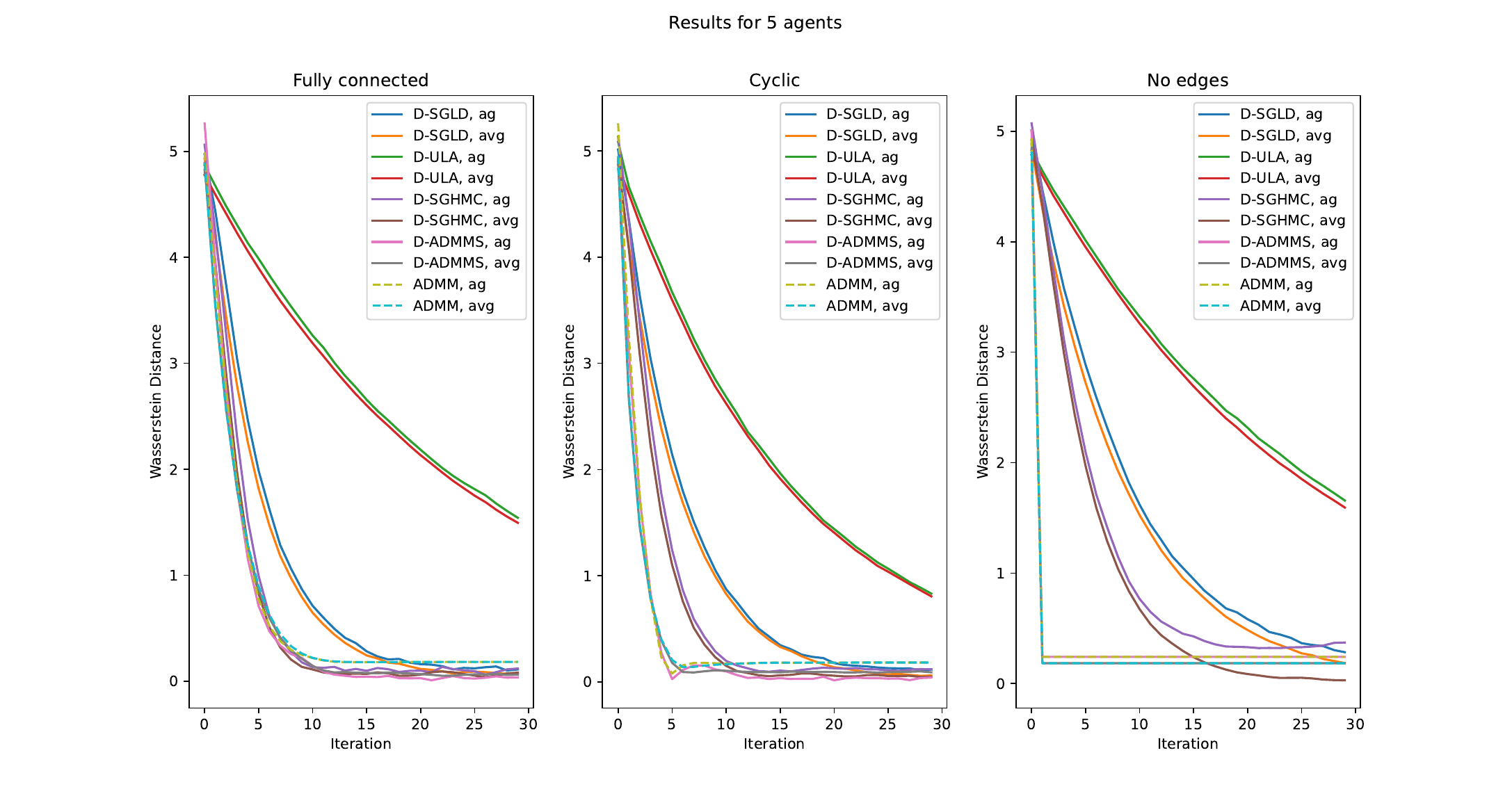}
         \label{fig:bay_lin_reg - num_agents [5] - num_samples [200]}
     \end{subfigure}
     \begin{subfigure}
         \centering
        \includegraphics[width=0.8\columnwidth, center]{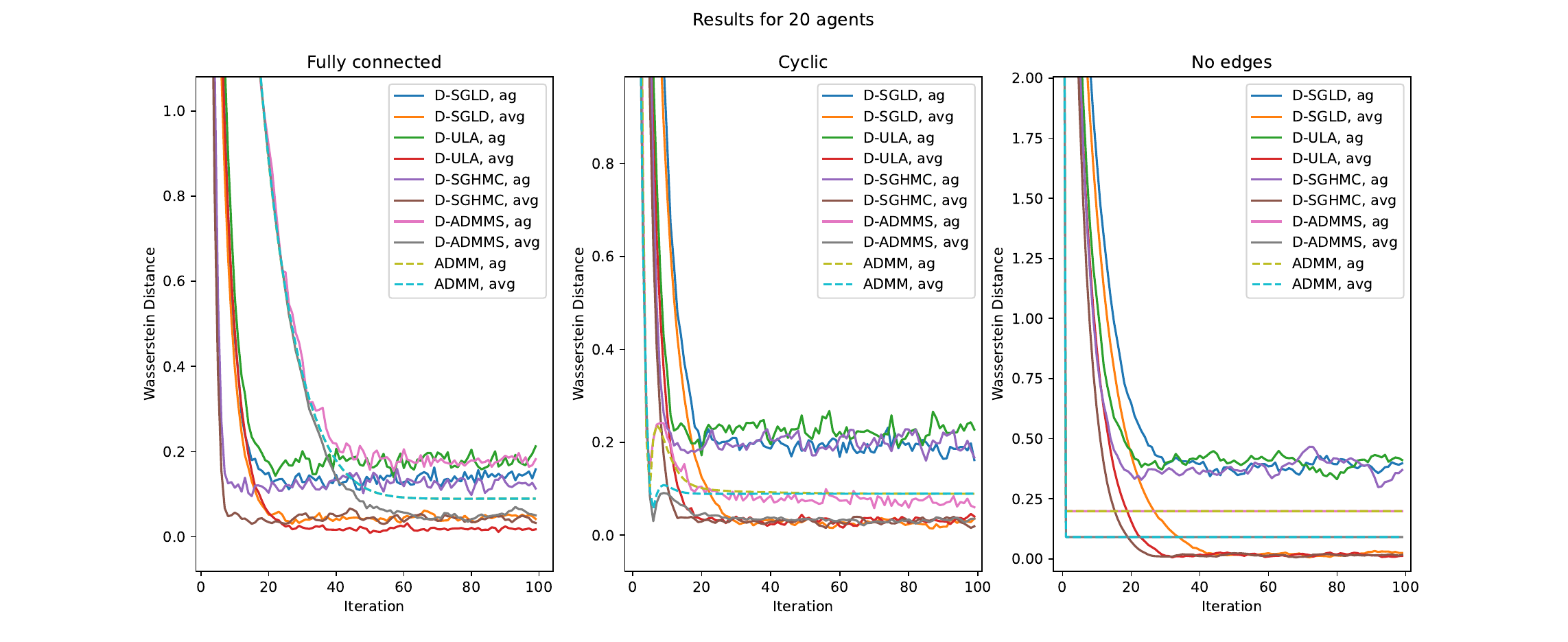}
         \label{fig:bay_lin_reg - num_agents [20] - num_samples [200]}
     \end{subfigure}
    \begin{subfigure}
         \centering
        \includegraphics[width=0.8\columnwidth, center]{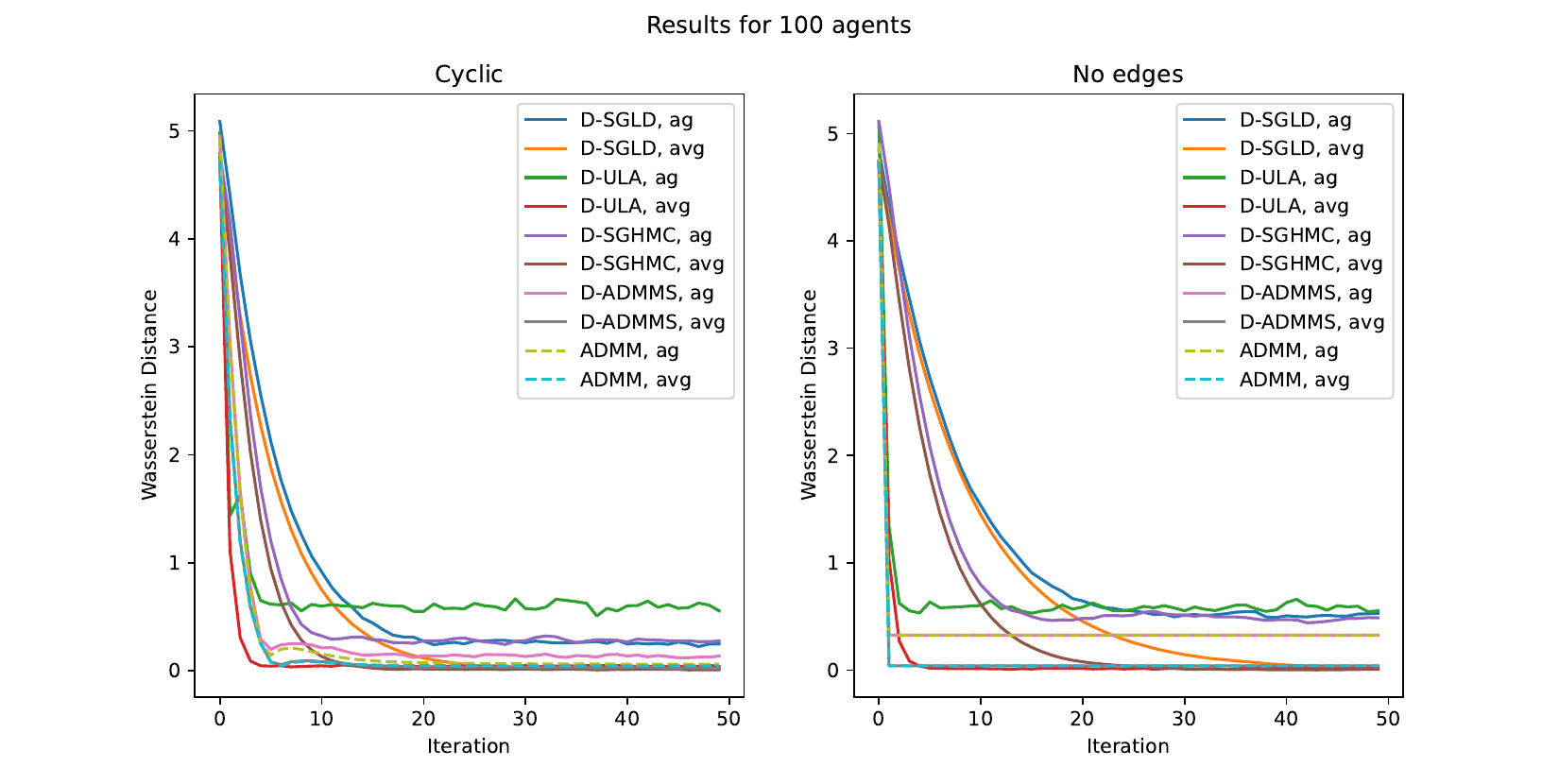}
         \label{fig:bay_lin_reg - num_agents [100] - num_samples [200]}
     \end{subfigure}
        \caption{$2$-Wasserstein distance to target distribution vs iteration for $n_i=200$. Both the distance to the target distribution of the average iterate (avg) and a specific agent iterate (ag) are provided for each method. For the sparsely connected (cyclic) graph topology, our proposed algorithm (D-ADMMS) outperforms the baselines (D-SGLD, D-ULA, D-SGHMC) in terms of Wasserstein distance between the distribution of the agent iterate and the target distribution.}
        \label{fig:bay_lin_reg_num_samples_[200]}
\end{figure}
\subsubsection{Ablation}
We study the robustness of the proposed scheme with respect to its hyper-parameter $\rho$. We note that D-ADMMS only requires setting the hyper-parameter $\rho$. In Figure \ref{fig:bay_ling_reg_ablation_rho}, we assume a network of $20$ agents with $n_i=50$ and we compute the $2$-Wasserstein distance between the empirical distribution of an agent's iterate, $x_i^{(k)}$, when D-ADMMS is deployed, and the true posterior. We deploy D-ADMMS with different values of $\rho$ in order to study the algorithm's robustness to the choice of hyper-parameter. We use $100$ trials in order to determine the iterate's empirical distribution.

Figure \ref{fig:bay_ling_reg_ablation_distr} shows the convergence in terms of $2$-Wasserstein distance between an agent's iterate, when D-ADMMS is deployed, and the true posterior for a cyclic network of $20$ agents with $n_i=50$. In this case, however, we fix $\rho=5$ and vary the initial probability distribution of the $100$ sample iterates. We perform $10$ experiments. For the first experiment, we assume that the initial distribution for each agent iterate is the standard normal ($x_i^{(0)} \sim \mathcal{N}(0, I),\ \forall i \in \mathcal{V}$). In each remaining experiment $q$, the $100$ sample iterates are initially drawn from the distribution $\mathcal{N}\left( (-1,2), \Sigma_q \right)$, where $\Sigma_q = A_q A_q^T$, and the entries of $A_q$ are drawn from the uniform $\mathcal{U}(0,10)$.
Figure \ref{fig:bay_ling_reg_sample_distr} demonstrates the evolution of samples in D-ADMMS for a cyclic topology of $5$ agents with $n_i=50$, $x_i{(0)} \sim \mathcal{N}(0, I)$, $\rho =5$, and $x \in \mathbb{R}^1$.

\begin{figure}[h]
     \centering
        \includegraphics[width=0.8\columnwidth, center]{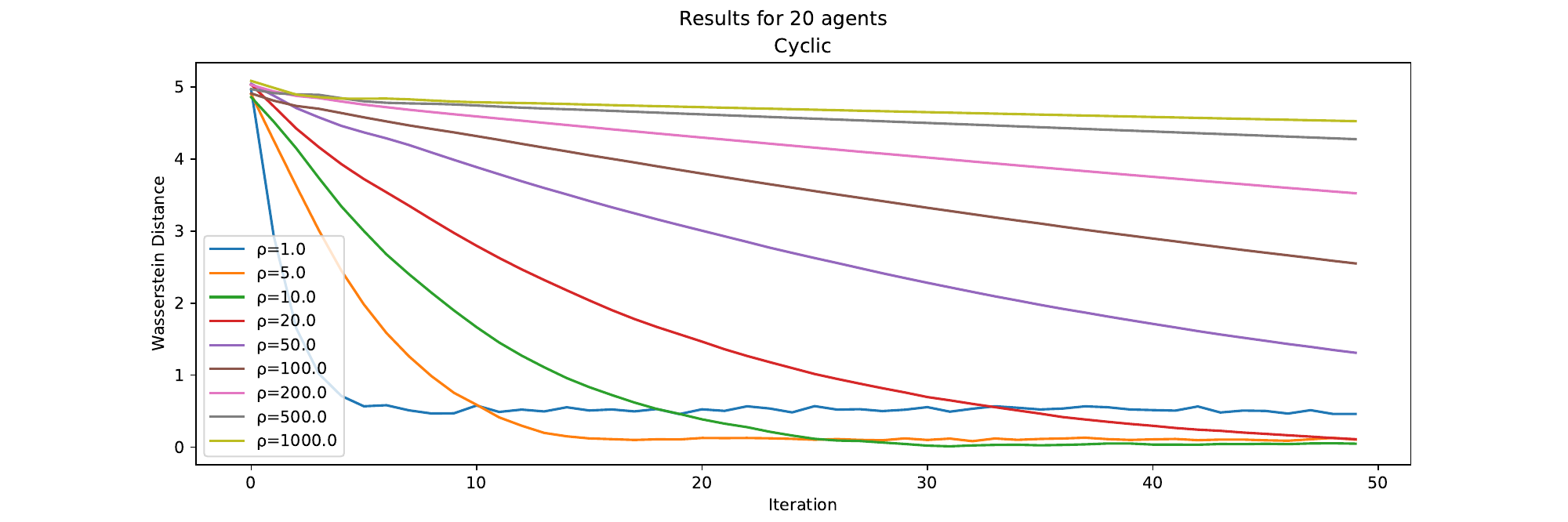}
      \caption{$2$-Wasserstein distance of an agent's iterate to the target distribution for varying $\rho$ in D-ADMMS.}
    \label{fig:bay_ling_reg_ablation_rho}
\end{figure}

\begin{figure}[h]
     \centering
        \includegraphics[width=0.8\columnwidth, center]{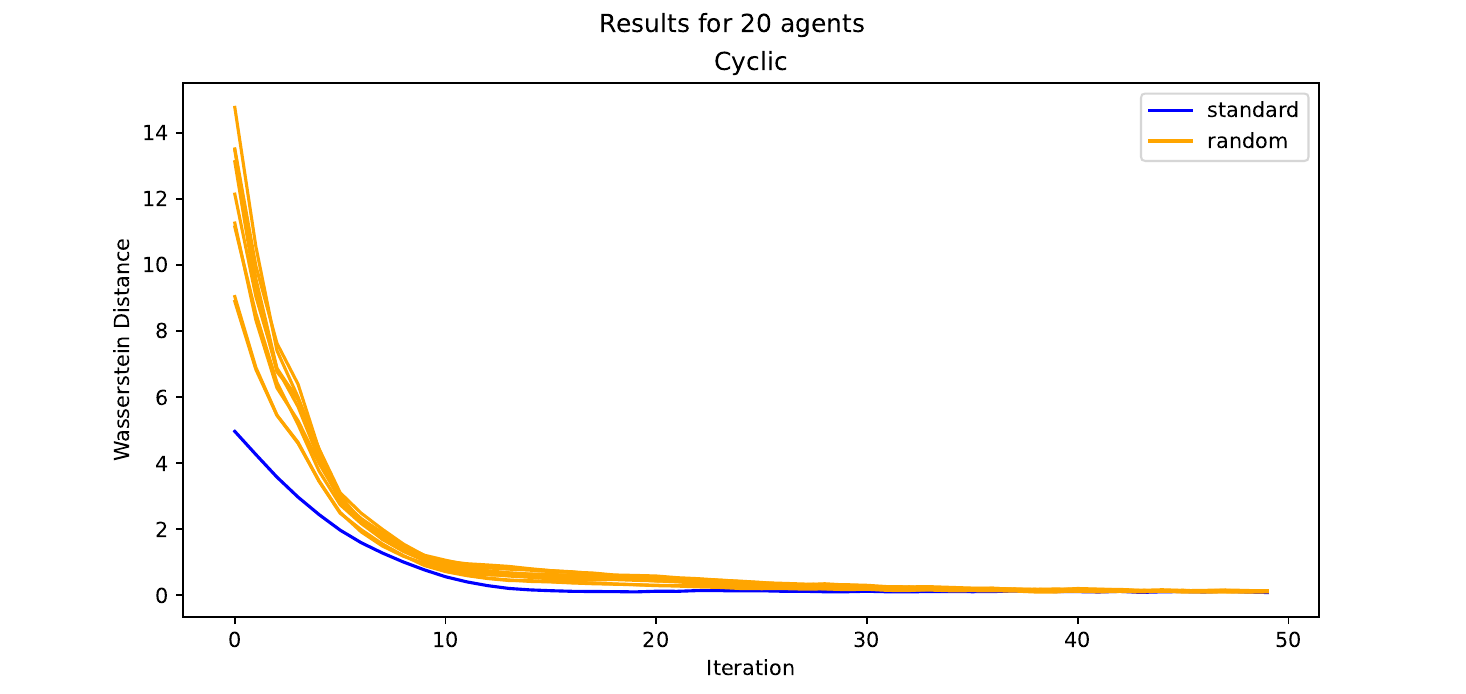}
      \caption{$2$-Wasserstein distance of an agent's iterate to the target distribution for varying initial sample distribution in D-ADMMS. Standard refers to $x_i{(0)} \sim \mathcal{N}(0, I)$.}
    \label{fig:bay_ling_reg_ablation_distr}
\end{figure}

\begin{figure}[h]
     \centering
        \includegraphics[width=0.8\columnwidth, center]{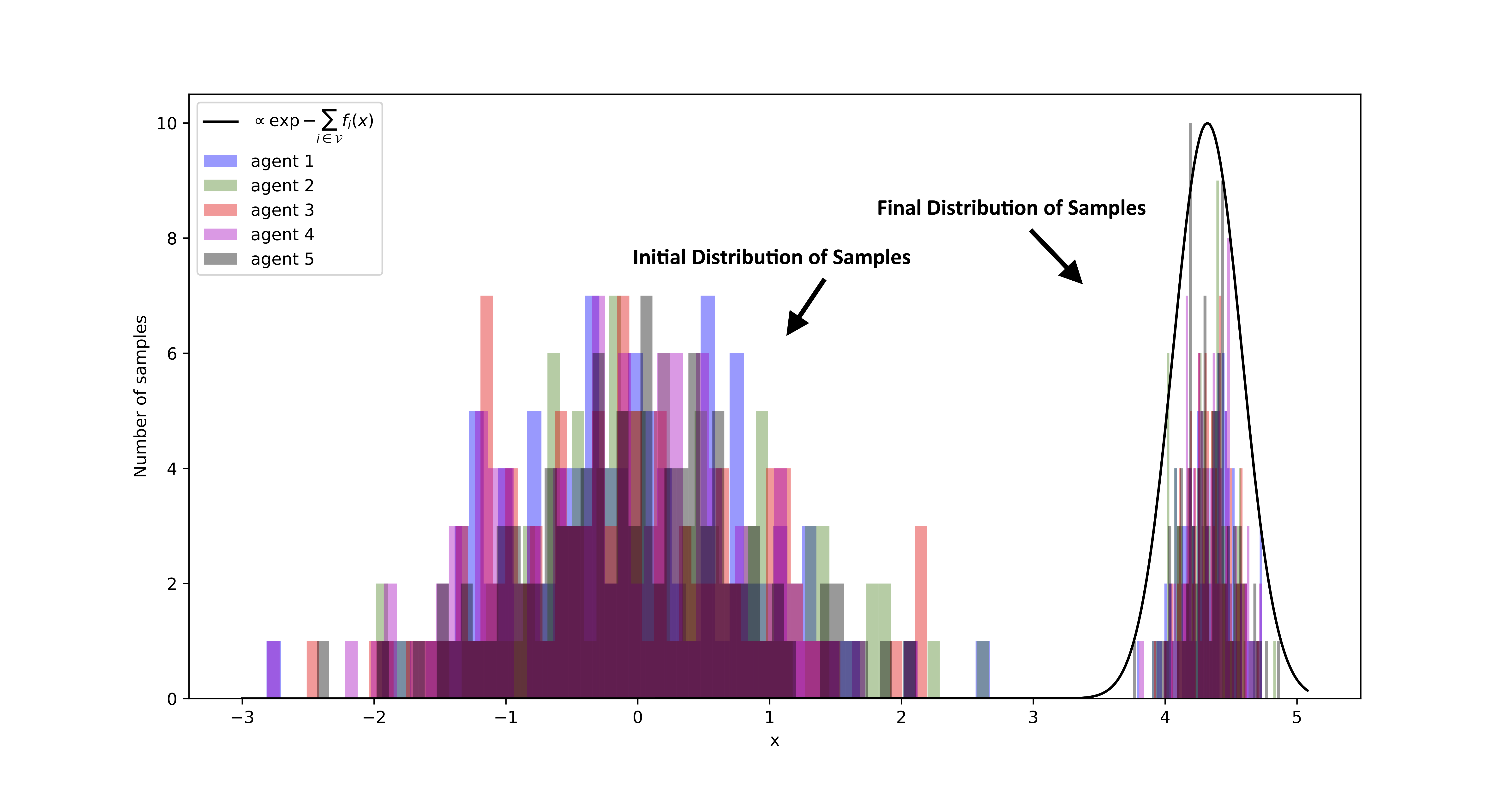}
      \caption{Evolution of the agents' sample distributions in D-ADMMS for a cyclic network of five agents. Each color corresponds to the samples of a different agent. We also include the true global posterior distribution up to a scaling factor.}
    \label{fig:bay_ling_reg_sample_distr}
\end{figure}

\subsection{Bayesian Logistic Regression}
\subsubsection{Problem}
We consider the distributed Bayesian logistic regression setting. We consider a varying number of agents ($N=5, 20, 50$) and a varying network topology (fully connected, ring-cyclic, and no-edge) in our results. We first create a dataset of IID $(z^l, y^l)$ pairs indexed by $l$, where $y^l$ is a binary label ($0$ or $1$). The likelihood function of a given sample for model parameters $x$ is
\begin{equation}\label{eq:likelihood-BayLogReg}
\mathbb{P}\left( y=1 \mid z; x \right) = \left(1+\exp{-x^Tz} \right)^{-1}.
\end{equation}
We assume $x \in \mathbb{R}^3$. The prior distribution over the model parameters is $\mathcal{N}(0, \lambda I)$, where $\lambda=10$.
A data point is created as follows:
\begin{equation}
    z^l \sim \mathcal{N}(0, 20I),\ p^l \sim \mathcal{U}(0,1),\\ 
\end{equation}
while the label $y^l$ is $1$ if $p^l \leq \left(1+\exp{-x^Tz^l} \right)^{-1}$  and $0$ otherwise.
The parameter $x$ is sampled from its prior distribution and $\mathcal{U}(0,1)$ denotes the uniform distribution between $0$ and $1$. Assume that each agent $i$ possesses $n_i$ independent data points $(z^l_i, y^l_i)$, where the first $\Tilde{n}_i$ data points are those with label $y_i^l=1$. Then the goal in Bayesian logistic regression is to sample from the global posterior, which is proportional to $\exp{-\sum_{i \in \mathcal{V}} f_i(x)}$, where
\begin{equation}
    f_i(x) = \sum_{l=1}^{n_i} \log \left(1+\exp{\psi_l x^Tz^l_i} \right) + \dfrac{\lVert x \rVert^2}{2\lambda N},
\end{equation}
and $\psi_l = -1$ if $1\leq l \leq \Tilde{n}_i$, while $\psi_l=1$ otherwise.
\subsubsection{Algorithm}
We use $n_i = 50$, $\rho=5$, $\eta_\mathrm{DSGLD}~= 0.0003$, $\eta_\mathrm{DSGHMC}\ = 0.02$, and $\gamma=30$ \citep{gurbuzbalaban2021decentralized}. For D-ULA, we assume $\alpha^{(k)}$ and $\zeta^{(k)}$ follow the same equations as in Section \ref{sec:bay_lin}.
We assume $\chi_1=\chi_2=0.05$ for the cyclic and no-edge topologies. For the case of the fully connected graph: when $N=5, 20$ we assume $\chi_1=0.55, \chi_2=0.05$, while for $N=50$ we set $\chi_1=\chi_2=0.9$. 

\begin{figure}[h!]
     \centering
     \begin{subfigure}
         \centering
        \includegraphics[width=0.78\columnwidth, center]{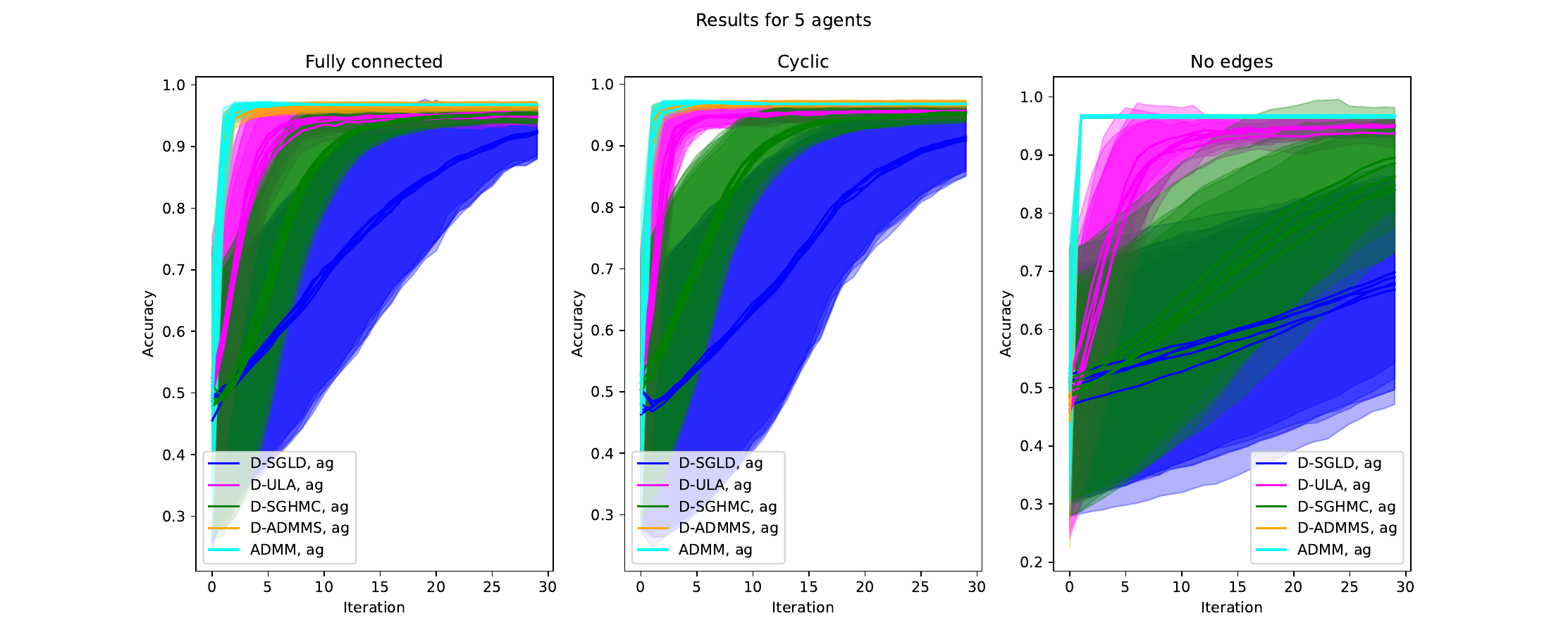}
        \label{fig:bay_log_reg - num_agents [5]}
     \end{subfigure}
     \begin{subfigure}
         \centering
        \includegraphics[width=0.78\columnwidth, center]{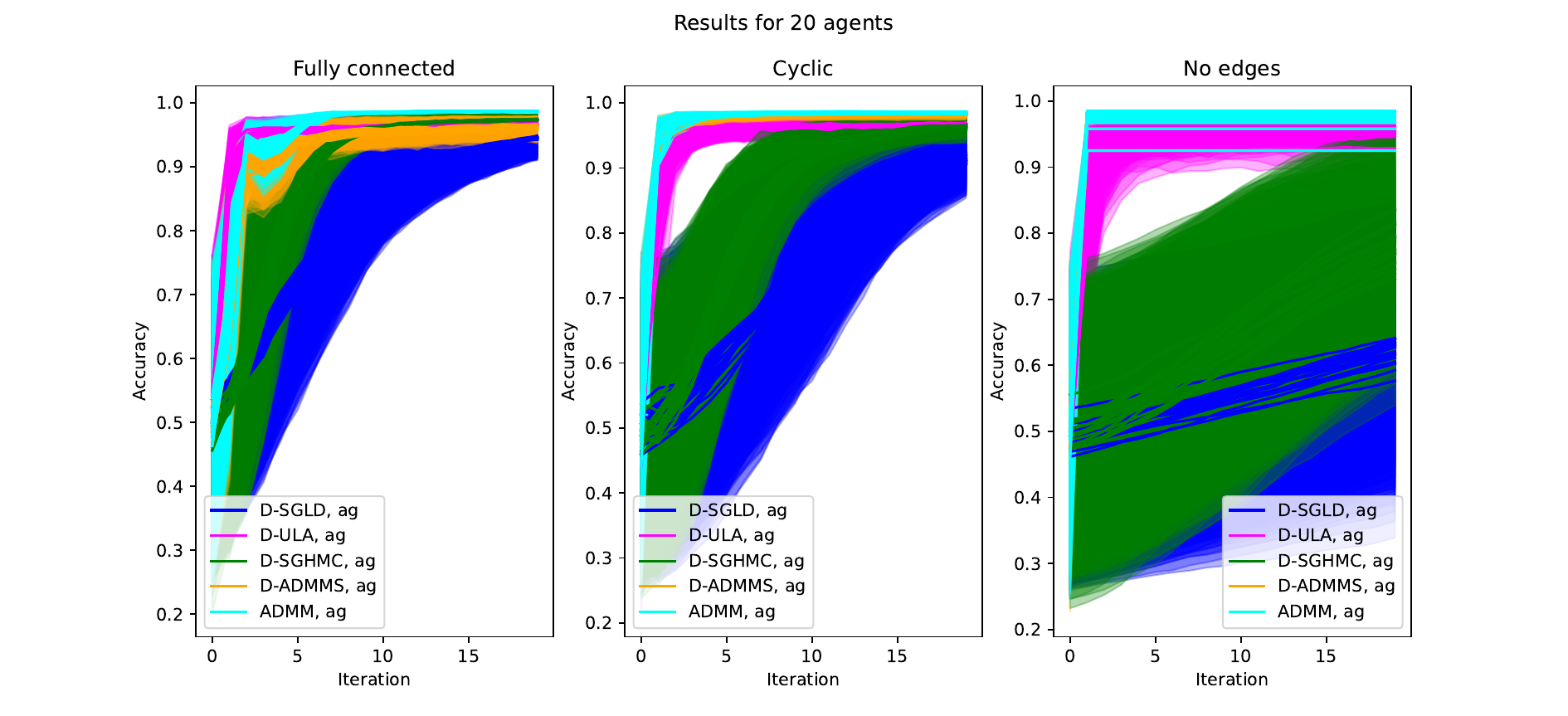}
        \label{fig:bay_log_reg - num_agents [20]}
     \end{subfigure}
    \begin{subfigure}
         \centering
        \includegraphics[width=0.78\columnwidth, center]{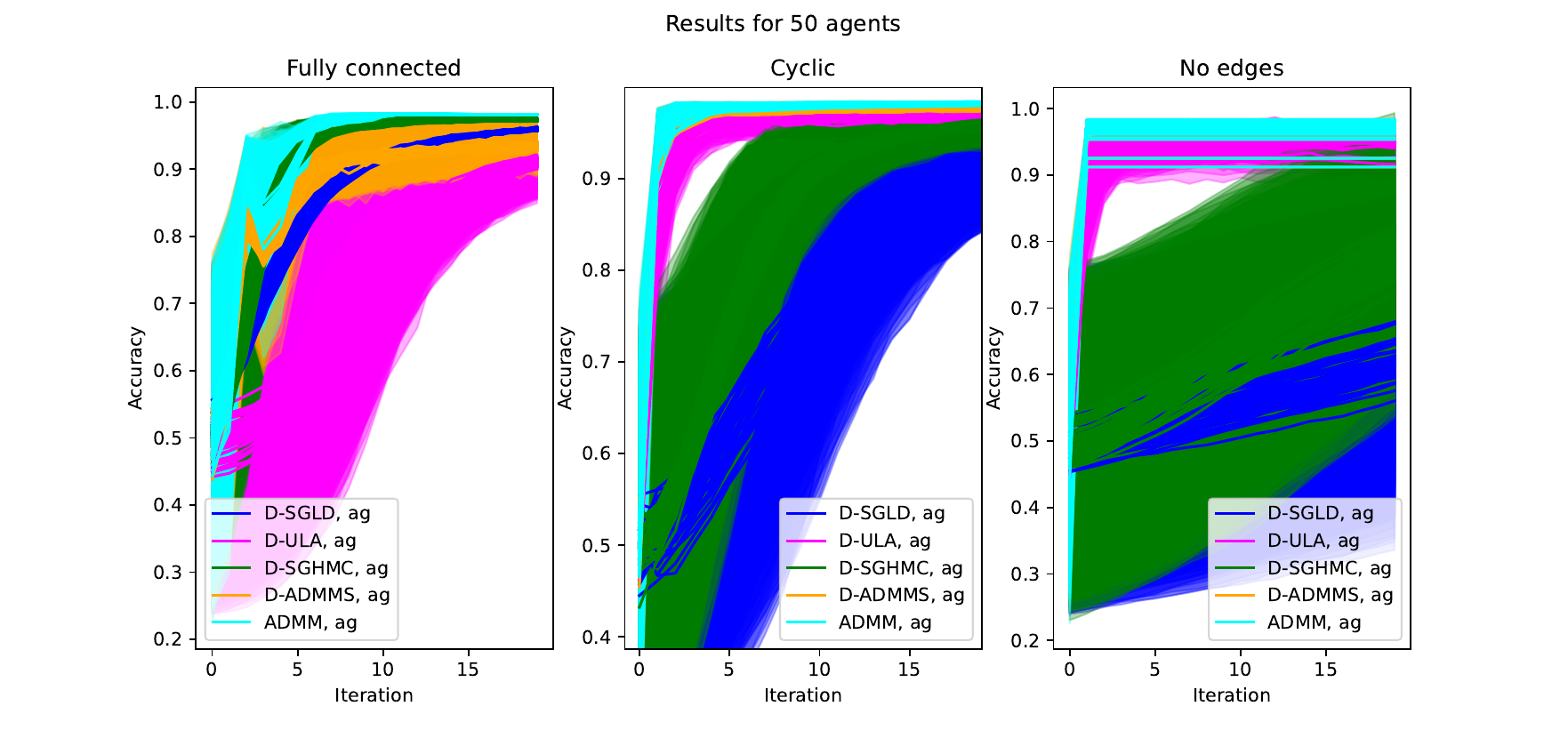}
        \label{fig:bay_log_reg - num_agents [50]}
     \end{subfigure}
      \caption{Prediction accuracy per iteration on complete dataset. We depict the accuracy per iteration on the complete dataset for each agent along with a 1-std interval over 100 independent trials for varying network topology. For the sparsely connected (cyclic) graph topology, our proposed algorithm (D-ADMMS) outperforms the baselines (D-SGLD, D-ULA, D-SGHMC) in terms of prediction accuracy.}
    \label{fig:bay_log_reg - num_samples [50]}
\end{figure}

\subsubsection{Discussion}
Because the posterior does not admit a Gaussian explicit formula, instead of $2$-Wasserstein distance, we provide results for prediction accuracy. In Figure \ref{fig:bay_log_reg - num_samples [50]}, we show the mean prediction accuracy on the total dataset, along with the $\pm$ $1$ standard deviation interval, of each agent, based on its iterate $x_i^{(k)}$, as a function of the iteration. If $\mathbb{P}(y=1 \mid z^l; x_i^{(k)}) \geq 0.5$, the agent assigns label $1$ to the data-point. The mean and standard deviation are computed through $100$ independent trials. We also include ADMM (i.e., $w_i^{(k)}=0$ in Algorithm \ref{alg:proposed}). 

We observe that D-ADMMS is able to obtain the highest accuracy out of all the sampling methods presented, in a smaller number of iterations for the cyclic topology. Its performance is similar to that of ADMM. ADMM possesses superior performance because it converges to the MAP model parameter, which is expected to have the highest accuracy. In addition, the deviation from the mean accuracy is smaller for D-ADMMS in the cyclic topology. For the no-edge topology, as in Bayesian linear regression, ADMM is identical to D-ADMMS. In the fully connected case for $N$ equal to $20$ and $50$ the decline in performance of D-ADMMS is similar to that found in the Bayesian linear regression task. \revzero{Overall, D-ADMMS performs worse as the number of agents and the connectivity of the graph increase. This behavior matches the observed performance in the Bayesian linear regression task.}

\section{Conclusions and Future Work}
\revzero{We proposed a novel distributed sampling algorithm based on ADMM. By using proximal updates to generate samples in a distributed manner for a log-concave distribution, our approach outperforms existing gradient-based algorithms. We have shown convergence of the distribution of the iterates of our algorithm to the target distribution and have demonstrated practical advantages for our method in regression problems. 
Despite these promising results, limitations of the proposed algorithm include synchronous and lossless communication among the agents of the network. 

Future directions include: i) analyzing the proposed scheme as the discretization of a stochastic differential equation to improve the convergence guarantees, ii) exploring connections between gradient flows and MCMC algorithms in distributed settings, iii) designing distributed sampling algorithms based on accelerated first-order or higher-order optimization methods, iv) considering distributed sampling with constrained support.}

\acks{The NASA University Leadership Initiative (grant $\#$ 80NSSC20M0163) provided funds to assist the first author with their research, but this article solely reflects the opinions and conclusions of its authors and not any NASA entity. For the first author, this work was also partially funded through the Alexander S. Onassis Foundation Scholarship program.}

\appendix

\newcommand{\refhere}{\textcolor{red}{\textbf{[(Reference here!)]}}}
\newcommand{\highlight}[1]{\textcolor{black}{#1}}
\newcommand{\rev}[1]{\textcolor{black}{#1}}

\newcommand{\cs}{[31]}

\section{Proof of Lemma 1}\label{sec:proof_lemma1}
We substitute 
\begin{equation}
     \frac{1}{2} M_{+}^T X^{(k)} - Z^{(k)} = 0
\end{equation}
into 
\begin{flalign}\label{eq:noisy_primal}
    \nabla f(X^{(k+1)}) + M_{-}\beta^{(k+1)} - \rho M_{+} \left(Z^{(k)} - Z^{(k+1)} \right) +\sqrt{2}D w^{(k+1)} = 0,\\
    \beta^{(k+1)} - \beta^{(k)} - \frac{\rho}{2} M_{-}^T X^{(k+1)} = 0,
\end{flalign}
to obtain
\begin{flalign}
    \nabla f(X^{(k+1)}) + M_{-}\beta^{(k+1)} - \frac{\rho}{2} M_{+} M_{+}^T \left(X^{(k)} - X^{(k+1)} \right) +\sqrt{2} D w^{(k+1)} = 0, \label{eq: lemma1_1}\\
    \beta^{(k+1)} - \beta^{(k)} - \frac{\rho}{2} M_{-}^T X^{(k+1)} = 0\label{eq: lemma1_2},
\end{flalign}
We observe that eq. (\ref{eq: lemma1_1}) depends on $M_{-}\beta^{(k+1)}$ rather than $\beta^{(k+1)}$. We thus multiply eq. (\ref{eq: lemma1_2}) by $M_{-}$ and obtain
\begin{equation}\label{eq: lemma1_3}
    M_{-}\beta^{(k+1)} - M_{-}\beta^{(k)} - \frac{\rho}{2} M_{-}M_{-}^T X^{(k+1)} = 0.
\end{equation}
We substitute eq. (\ref{eq: lemma1_3}) into eq. (\ref{eq: lemma1_1}) and get
\begin{equation}\label{eq:lemma1_100}
    \nabla f(X^{(k+1)}) + M_{-}\beta^{(k)} - \frac{\rho}{2} M_{+} M_{+}^T X^{(k)} + \left( \frac{\rho}{2} M_{+} M_{+}^T + \frac{\rho}{2} M_{-}M_{-}^T \right) X^{(k+1)}  +\sqrt{2} D w^{(k+1)} = 0.
\end{equation}
We define $p^{(k)}=M_{-}\beta^{(k)}$ and from eq. (\ref{eq: lemma1_3}, \ref{eq:lemma1_100}) we obtain the equivalent updates
\begin{flalign}
    \nabla f(X^{(k+1)}) + p^{(k)} - \frac{\rho}{2} M_{+} M_{+}^T X^{(k)} + \left( \frac{\rho}{2} M_{+} M_{+}^T + \frac{\rho}{2} M_{-}M_{-}^T \right) X^{(k+1)}  +\sqrt{2} D w^{(k+1)} = 0,\label{eq: lemma1_4}\\
    p^{(k+1)} - p^{(k)} - \frac{\rho}{2} M_{-}M_{-}^T X^{(k+1)} = 0\label{eq: lemma1_5}.
\end{flalign}
Notice that $L_{+}=\frac{1}{2}M_{+}M_{+}^T, L_{-}=\frac{1}{2}M_{-}M_{-}^T$ and $D = \dfrac{1}{2}\left( L_{+} + L_{-} \right)$. Hence eq. (\ref{eq: lemma1_4}-\ref{eq: lemma1_5}) can be written as
\begin{flalign}
    \nabla f(X^{(k+1)}) + p^{(k)} - \rho L_{+} X^{(k)} + 2\rho D X^{(k+1)}  +\sqrt{2} D w^{(k+1)} = 0,\label{eq: lemma1_6}\\
    p^{(k+1)} - p^{(k)} - \rho L_{-} X^{(k+1)} = 0\label{eq: lemma1_7}.
\end{flalign}
We have thus established that eq. (10-12) imply eq. (\ref{eq: lemma1_6}-\ref{eq: lemma1_7}). By a simple inspection, we can observe that eq. (\ref{eq: lemma1_6}-\ref{eq: lemma1_7}) are equivalent to the update equations of D-ADMMS in Algorithm 1, where $X^{(k)}$ is the concatenation of the $x_i^{(k)}$ from Algorithm \ref{alg:proposed}.

\section{Proof of Lemma 2}\label{sec:proof_lemma2}
\rev{Let us start with the following auxiliary lemma.}
\begin{lemma}
  Assume $x, y \in \mathbb{R}^n$. Then the inequality
  \begin{equation}\label{eq:aux}
      \|x+y\|^2 + (\kappa-1)\|x\|^2 \geq \left(1 - \frac{1}{\kappa}\right) \|y\|^2
  \end{equation}
holds for any $\kappa >1$.
\end{lemma}
\begin{proof}
    By expanding the left-hand side (LHS) of \eqref{eq:aux}, we obtain
    \begin{equation*}
        \lVert x \rVert^2 + 2 \langle x, y \rangle + \lVert y \rVert^2 + \left(\kappa -1\right) \lVert x \rVert^2 \geq  \left(1 - \frac{1}{\kappa}\right) \|y\|^2.
    \end{equation*}
    By bringing all terms to the LHS, we get
    \begin{equation*}
        \kappa \lVert x \rVert^2 + 2 \langle x, y \rangle + \frac{1}{\kappa}\lVert y \rVert^2  \geq  0.
    \end{equation*}
    By multiplying both sides by $\kappa$ and noticing that $\kappa>1>0$, we get
    \begin{equation*}
         \lVert \kappa x \rVert^2 + 2 \langle \kappa x, y \rangle + \lVert y \rVert^2  \geq  0,
    \end{equation*}
    which can be written as
    \begin{equation*}
        \lVert \kappa x +y \rVert^2 \geq 0,
    \end{equation*}
    this concluding the proof of the lemma.
\end{proof}

\highlight{Our proof strategy leverages equations (10-12) and the KKT conditions in (13-15), which yield the relations:}
\begin{equation} \label{eq:lemma2_0}
    \nabla f(X^{(k+1)}) - \nabla f(X^{*}) =  \rho M_{+} \left(Z^{(k)} - Z^{(k+1)} \right) - M_{-}\left( \beta^{(k+1)} - \beta^* \right)-\sqrt{2} D w^{(k+1)},
\end{equation}
\begin{equation}\label{eq:lemma2_1}
    \beta^{(k+1)} - \beta^{(k)} = \frac{\rho}{2} M_{-}^T \left( X^{(k+1)} -X^* \right),
\end{equation}
\begin{equation}\label{eq:lemma2_2}
     Z^{(k+1)} -Z^* = \frac{1}{2} M_{+}^T \left( X^{(k+1)} - X^* \right).
\end{equation}
\rev{We now exploit the relations in eq. (\ref{eq:lemma2_0}-\ref{eq:lemma2_2}) to show inequalities that hold for the iterates generated by Algorithm 1. To this end, we split the analysis into five steps, which are presented in the sequel.}
\subsection*{\highlight{Step 1: A basic inequality}}

\highlight{Since function $f(X)$ is strongly convex (under Assumption 1), we obtain}
\begin{multline}
    m_f \lVert X^{(k+1)} - X^* \rVert^2 \leq \langle X^{(k+1)} - X^*, \nabla f(X^{(k+1)}) -  \nabla f(X^*) \rangle =\\
    \langle X^{(k+1)} - X^*, \rho M_{+} \left(Z^{(k)} - Z^{(k+1)} \right) - M_{-}\left( \beta^{(k+1)} - \beta^* \right)-\sqrt{2} D w^{(k+1)}\rangle =\\
    \langle X^{(k+1)} - X^*, \rho M_{+} \left(Z^{(k)} - Z^{(k+1)} \right) \rangle \\
    + \langle X^{(k+1)} - X^*, - M_{-}\left( \beta^{(k+1)} - \beta^* \right)\rangle + \langle X^{(k+1)} - X^*, -\sqrt{2} D w^{(k+1)}\rangle = \\
    \rho \langle M_{+}^T \left( X^{(k+1)} - X^*\right),  Z^{(k)} - Z^{(k+1)} \rangle \\
    - \langle M_{-}^T \left(X^{(k+1)} - X^*\right),  \beta^{(k+1)} - \beta^* \rangle - \langle X^{(k+1)} - X^*, \sqrt{2} D w^{(k+1)}\rangle.
    \label{eq:theo-proof-main-ineq}
\end{multline}
\highlight{The first equality holds due to eq. \eqref{eq:lemma2_0}, the second and thrid ones hold by simply expanding the terms and performing trivial algebraic manipulations.} We substitute eq. (\ref{eq:lemma2_1}-\ref{eq:lemma2_2}) into the last equation of \eqref{eq:theo-proof-main-ineq} to \highlight{obtain}
\begin{multline}
    m_f \lVert X^{(k+1)} - X^* \rVert^2 \leq \\
    2\rho \langle Z^{(k+1)} - Z^*, Z^{(k)} - Z^{(k+1)} \rangle + \frac{2}{\rho} \langle \beta^{(k)} - \beta^{(k+1)}, \beta^{(k+1)} - \beta^* \rangle - \langle X^{(k+1)} - X^*, \sqrt{2} D w^{(k+1)}\rangle = \\
    2\left( U^{(k)} - U^{(k+1)} \right)^T G\left( U^{(k+1)} - U^*\right)-\langle X^{(k+1)} - X^*,\sqrt{2} D w^{(k+1)}\rangle,
    \label{eq:theo-proof-main-ineq-2}
\end{multline}
\highlight{where we recall that $U^{(k)} = (Z^{(k)}, \beta^{(k)})$ and $G= \mathrm{diag}\left\{\rho I_{2 \mid \mathcal{E} \mid d}, \frac{1}{\rho}I_{2 \mid \mathcal{E} \mid d}\right\}$. Using the relation $\inner[G]{a-b}{b-c} = \|a-c\|_G^2 - \|a-b\|_G^2 - \|b-c\|_G^2$ for the first term on the right-hand side of eq. \eqref{eq:theo-proof-main-ineq-2}, with $a=U^{(k)}$, $b=U^{(k+1)}$, and $c=U^*$, we obtain}
\begin{multline}\label{eq:lemma2_3}
    m_f \lVert X^{(k+1)} - X^* \rVert^2 \leq \\
    \lVert U^{(k)} - U^* \rVert_G^2 - \lVert U^{(k+1)} - U^* \rVert_G^2 - \lVert U^{(k)} - U^{(k+1)} \rVert_G^2 - \langle X^{(k+1)} - X^*, \sqrt{2} D w^{(k+1)}\rangle.
\end{multline}
We now upper bound the last term
\begin{multline}
    -\langle X^{(k+1)} - X^*, \sqrt{2} D w^{(k+1)}\rangle \leq \lVert X^{(k+1)} - X^* \rVert \lVert \sqrt{2} D w^{(k+1)} \rVert \\
    \leq \frac{1}{2} \left(\lVert X^{(k+1)} - X^* \rVert + \lVert \sqrt{2} D w^{(k+1)} \rVert \right)^2,
\end{multline}
\highlight{where we apply the inequality $ ab  \leq \frac{1}{2} ( a+b)^2$ in the last step. This yields}
\begin{multline}\label{eq:lemma2_7}
    m_f \lVert X^{(k+1)} - X^* \rVert^2 \leq 
    \lVert U^{(k)} - U^* \rVert_G^2 \\ - \lVert U^{(k+1)} - U^* \rVert_G^2 - \lVert U^{(k)} - U^{(k+1)} \rVert_G^2 + \frac{1}{2} \left(\lVert X^{(k+1)} - X^* \rVert + \lVert \sqrt{2} D w^{(k+1)} \rVert \right)^2.
\end{multline}
\highlight{Eq. \eqref{eq:lemma2_7} is the basic inequality of our analysis. The next steps of the proof constitute in further developing such inequality.}
\subsection*{\highlight{Step 2: Dealing with the term} $\lVert U^{(k)} - U^{(k+1)} \rVert_G^2 + m_f \lVert X^{(k+1)} - X^* \rVert^2$}
\highlight{Observe that}
\begin{multline}
    \lVert U^{(k)} - U^{(k+1)} \rVert_G^2 + m_f \lVert X^{(k+1)} - X^* \rVert^2 =\\
    \rho \lVert Z^{(k)} - Z^{(k+1)} \rVert^2 + \frac{1}{\rho} \lVert \beta^{(k)} - \beta^{(k+1)} \rVert^2 + m_f \lVert X^{(k+1)} - X^* \rVert^2.
    \label{eq:theo-proof-step-2}
\end{multline}
\highlight{We now focus on obtaining a lower bound for the right-hand side (RHS) term in eq. \eqref{eq:theo-proof-step-2}, which will be done in two steps as described below.}
\subsubsection*{\highlight{Step 2.1: An intermediate inequality}}
\highlight{We show that the inequality}
\begin{multline}
    \frac{\rho \kappa \sigma^2_{\mathrm{max}}(M_{+})}{(\kappa -1) \sigma^2_{\mathrm{min}}(M_{-})} \lVert Z^{(k+1)} - Z^{(k)} \rVert^2 +  \frac{\kappa M_f^2}{\rho \sigma^2_{\mathrm{min}}(M_{-})} \lVert X^{(k+1)} - X^{*} \rVert^2 \geq \\
     \frac{1}{\rho} \lVert \beta^{(k+1)} - \beta^* \rVert^2 - \frac{2\sqrt{2}}{\rho\sigma^2_{\mathrm{min}}(M_{-})} \lVert M_{-} \left(\beta^{(k+1)} - \beta^* \right) \rVert \lVert Dw^{(k+1)} \rVert +\\
     \frac{2}{\rho \sigma^2_{\mathrm{min}}(M_{-})} \lVert D w^{(k+1)} \rVert^2,
     \label{eq:theo-proof-intermediate-ineq-2}
\end{multline}
\rev{holds, for any $\kappa > 1$}. \highlight{The first step to obtain inequality \eqref{eq:theo-proof-intermediate-ineq-2} is to manipulate the relation in \eqref{eq:lemma2_0} by means of the inequality $\|a+b\|^2 + (\kappa-1)\|a\|^2 \geq \left(1 - \frac{1}{\kappa}\right) \|b\|^2$ (see Lemma 2 for a proof of this inequality), which holds for $\kappa > 1$. Indeed, setting $a = \nabla f(X^{(k+1)}) - \nabla f(X^\star)$ and $b = M_-\left( \beta^{(k+1)} - \beta^\star\right) + \sqrt{2} D w^{(k+1)}$, we obtain}
\begin{multline}
    \left(1-\frac{1}{\kappa}\right) \lVert M_{-} \left( \beta^{(k+1)} - \beta^*\right) + \sqrt{2}Dw^{(k+1)} \rVert^2 \\ \leq \lVert \rho M_{+} \left(Z^{(k)} - Z^{(k+1)} \right) \rVert^2 + (\kappa-1) \lVert \nabla f(X^{(k+1)}) - \nabla f(X^*)\rVert^2  \\
    \leq \rho^2 \sigma_\mathrm{max}^2(M_{+}) \lVert Z^{(k+1)}-Z^{(k)} \rVert^2 +  (\kappa-1) M_f^2 \lVert X^{(k+1)} - X^* \rVert^2,
    \label{eq:theo-proof-intermediate-ineq-3}
\end{multline}
\highlight{where the first inequality follows from $\|a+b\| = \|\rho M_+(Z^{(k)} - Z^{(k+1)}\|$ due to \eqref{eq:lemma2_0}, and the second inequality follows by Lipschitz continuity of $\nabla f(X)$ and the fact that} \rev{the largest singular value of $M_+$ equals the induced $2$-norm $\lVert M_{+} \rVert = \max_{x \neq 0} \frac{\lVert M_{+} x \rVert}{\lVert x \rVert}$}. \highlight{By expanding the squares in the LHS of \eqref{eq:theo-proof-intermediate-ineq-3} and using the inequality $\inner[]{a}{b} \geq - \|a\|\|b\|$, we have that}
\begin{multline}\label{eq:lemma2_4}
    \left(1-\frac{1}{\kappa}\right) \left( \lVert M_{-} \left( \beta^{(k+1)} - \beta^*\right) \rVert^2 -2 \lVert M_{-} \left( \beta^{(k+1)} - \beta^*\right) \rVert \lVert \sqrt{2}Dw^{(k+1)} \rVert + \lVert \sqrt{2}Dw^{(k+1)} \rVert^2 \right) \\
    \leq \rho^2 \sigma_\mathrm{max}^2(M_{+}) \lVert Z^{(k+1)}-Z^{(k)} \rVert^2 + (\kappa-1) M_f^2 \lVert X^{(k+1)} - X^* \rVert^2.
\end{multline}
\highlight{We now use the fact that $\lVert M_{-} \left( \beta^{(k+1)} - \beta^* \right) \rVert^2 \geq \sigma_\mathrm{min}^2(M_{-}) \lVert \beta^{(k+1)} - \beta^* \rVert^2$} \rev{(because both $\beta^*$ and $\beta^{(k+1)}$ lie in the column space of $M_{-}^T$)} \highlight{and then multiply the resulting inequality by $\frac{\kappa}{\rho(\kappa-1)\sigma^2_{\mathrm{min}}(M_{-})}$ to obtain eq. \eqref{eq:theo-proof-intermediate-ineq-2}.}

\subsubsection*{\highlight{Step 2.2: A trivial inequality}}

Notice that from eq. (\ref{eq:lemma2_2})
\begin{equation}
    \lVert Z^{(k+1)} - Z^* \rVert^2 = \frac{1}{4} \lVert M_{+}^T \left( X^{(k+1)} - X^*\right)\rVert^2 \leq \frac{1}{4} \sigma_\mathrm{max}^2(M_{+}) \lVert X^{(k+1)}-X^* \rVert^2.
\end{equation}
\rev{By simple multiplication of both sides by $\rho$, we obtain}
\begin{equation}\label{eq:lemma2_5}
    \rho\lVert Z^{(k+1)} - Z^* \rVert^2  \leq \frac{\rho \sigma_\mathrm{max}^2(M_{+})}{4} \lVert X^{(k+1)}-X^* \rVert^2.
\end{equation}

\subsubsection*{\highlight{Step 2.3: Combining the results from steps 2.1 and 2.2 to obtain \eqref{eq:theo-proof-intermediate-ineq-2}}}

We add inequality \eqref{eq:lemma2_5} into (\ref{eq:theo-proof-intermediate-ineq-2}) to obtain
\begin{multline}\label{eq:lemma2_6}
    \frac{\rho \kappa \sigma^2_{\mathrm{max}}(M_{+})}{(\kappa -1) \sigma^2_{\mathrm{min}}(M_{-})} \lVert Z^{(k+1)} - Z^{(k)} \rVert^2 + \left( \frac{\kappa M_f^2}{\rho \sigma^2_{\mathrm{min}}(M_{-})} + \frac{\rho}{4} \sigma^2_{\mathrm{max}}(M_{+})\right) \lVert X^{(k+1)} - X^{*} \rVert^2 \geq \\
    \rho \lVert Z^{(k+1)} - Z^* \rVert^2 +\frac{1}{\rho} \lVert \beta^{(k+1)} - \beta^* \rVert^2 - \frac{2\sqrt{2}}{\sigma^2_{\mathrm{min}}(M_{-})} \lVert M_{-} \left(\beta^{(k+1)} - \beta^* \right) \rVert \lVert Dw^{(k+1)} \rVert +\\
     \frac{2}{\rho \sigma^2_{\mathrm{min}}(M_{-})} \lVert D w^{(k+1)} \rVert^2.
\end{multline}
We now let
\begin{equation}
    \delta = \min \Biggl \{ \frac{(\kappa -1) \sigma^2_{\mathrm{min}}(M_{-})}{\kappa \sigma^2_{\mathrm{max}}(M_{+})}, \frac{m_f}{\frac{\rho}{4} \sigma^2_{\mathrm{max}}(M_{+}) + \frac{\kappa M_f^2}{\rho \sigma^2_{\mathrm{min}}(M_{-})} }\Biggr \}>0,
\end{equation}
\rev{and notice that if $\delta = \frac{(\kappa -1) \sigma^2_{\mathrm{min}}(M_{-})}{\kappa \sigma^2_{\mathrm{max}}(M_{+})}$, then $\delta \left( \frac{\kappa M_f^2}{\rho \sigma^2_{\mathrm{min}}(M_{-})} + \frac{\rho}{4} \sigma^2_{\mathrm{max}}(M_{+})\right) \leq m_f$. Similarly, if $\delta = \frac{m_f}{\frac{\rho}{4} \sigma^2_{\mathrm{max}}(M_{+}) + \frac{\kappa M_f^2}{\rho \sigma^2_{\mathrm{min}}(M_{-})}}$, then $\delta \frac{\rho \kappa \sigma^2_{\mathrm{max}}(M_{+})}{(\kappa -1) \sigma^2_{\mathrm{min}}(M_{-})} \leq \rho$.}
Therefore, by multiplying eq. (\ref{eq:lemma2_6}) with $\delta$, we obtain
\begin{multline}
    \rho \lVert Z^{(k+1)} - Z^{(k)} \rVert^2 +  m_f \lVert X^{(k+1)} - X^{*} \rVert^2 \geq \\
    \rho \delta \lVert Z^{(k+1)} - Z^* \rVert^2 +\frac{\delta}{\rho} \lVert \beta^{(k+1)} - \beta^* \rVert^2 - \frac{2\sqrt{2}\delta}{\sigma^2_{\mathrm{min}}(M_{-})} \lVert M_{-} \left(\beta^{(k+1)} - \beta^* \right) \rVert \lVert Dw^{(k+1)} \rVert +\\
     \frac{2\delta}{\rho \sigma^2_{\mathrm{min}}(M_{-})} \lVert D w^{(k+1)} \rVert^2.
\end{multline}
\rev{We add the positive term $\frac{1}{\rho} \lVert \beta^{(k+1)} - \beta^{(k)} \rVert^2$ to the LHS of the last equation and apply the definition of $\lVert U^{(k+1)} - U^{(k)} \rVert_G^2$ to get}
\begin{multline}\label{eq:lemma2_1001}
    \lVert U^{(k+1)} - U^{(k)} \rVert_G^2 +  m_f \lVert X^{(k+1)} - X^{*} \rVert^2 \geq \\
    \delta \lVert U^{(k+1)} - U^* \rVert_G^2  - \frac{2\sqrt{2}\delta}{\sigma^2_{\mathrm{min}}(M_{-})} \lVert M_{-} \left(\beta^{(k+1)} - \beta^* \right) \rVert \lVert Dw^{(k+1)} \rVert +\\
     \frac{2\delta}{\rho \sigma^2_{\mathrm{min}}(M_{-})} \lVert D w^{(k+1)} \rVert^2.
\end{multline}
\rev{The last equation is a lower bound for the term $\lVert U^{(k)} - U^{(k+1)} \rVert_G^2 + m_f \lVert X^{(k+1)} - X^* \rVert^2$, thus concluding step 2 of the proof.}

\subsection*{\rev{Step 3: Manipulating the lower bound \eqref{eq:lemma2_1001} of $\lVert U^{(k)} - U^{(k+1)} \rVert_G^2 + m_f \lVert X^{(k+1)} - X^* \rVert^2$}}

We combine inequality \eqref{eq:lemma2_1001} with eq. (\ref{eq:lemma2_7}) and obtain
\begin{multline}\label{eq:lemma2_8}
    \lVert U^{(k+1)} - U^* \rVert_G^2 \leq \frac{1}{1+\delta} \lVert U^{(k)} - U^* \rVert_G^2 + \frac{1}{2(1+\delta)} \left(\lVert X^{(k+1)} - X^* \rVert + \lVert \sqrt{2} D w^{(k+1)} \rVert \right)^2 +\\
    \frac{2\sqrt{2}\delta}{(1+\delta)\sigma^2_{\mathrm{min}}(M_{-})} \lVert M_{-} \left(\beta^{(k+1)} - \beta^* \right) \rVert \lVert Dw^{(k+1)} \rVert - \frac{2\delta}{(1+\delta)\rho \sigma^2_{\mathrm{min}}(M_{-})} \lVert D w^{(k+1)} \rVert^2,
\end{multline}
by using $a \leq b\leq c \Rightarrow a \leq c$, where $b$ stands for $\lVert U^{(k+1)} - U^{(k)} \rVert_G^2 +  m_f \lVert X^{(k+1)} - X^{*} \rVert^2$, and some algebraic manipulations.
\vspace{0.3cm}

\subsubsection*{\highlight{Step 3.1: An intermediate trick}}

Eq. (\ref{eq:lemma2_3}) also gives us the upper bound
\begin{equation}
     \lVert X^{(k+1)} - X^* \rVert^2 \leq \frac{1}{m_f} \lVert U^{(k)} - U^* \rVert_G^2 + \frac{1}{m_f}\langle X^{(k+1)} - X^*, -\sqrt{2} D w^{(k+1)}\rangle,
\end{equation}
which, by completing the square, is equivalently written as
\begin{equation}\label{eq:lemma2_13}
    \lVert X^{(k+1)} - X^* + \frac{1}{\sqrt{2}m_f} D w^{(k+1)} \rVert^2 \leq \frac{1}{m_f} \lVert U^{(k)} - U^* \rVert_G^2 + \frac{1}{2m_f^2} \lVert D w^{(k+1)} \rVert^2.
\end{equation}

\noindent We now work to reform the second and third terms on the (RHS) of eq. (\ref{eq:lemma2_8}). 
\vspace{0.3cm}

\subsubsection*{\highlight{Step 3.2: Manipulating the third term on the RHS of eq. (\ref{eq:lemma2_8})}}
\rev{We first manipulate the third term of eq. \eqref{eq:lemma2_8}. From eq. (\ref{eq:lemma2_1}) we have that}
\begin{equation}\label{eq:lemma2_1002}
    \beta^{(k+1)} - \beta^* = \beta^{(k)} - \beta^* + \frac{\rho}{2} M_{-}^T \left( X^{(k+1)} -X^* \right).
\end{equation}
which gives
\begin{multline}\label{eq:lemma2_9}
    \lVert M_{-} \left( \beta^{(k+1)} - \beta^* \right) \rVert \leq \sigma_\mathrm{max}(M_{-}) \lVert \beta^{(k+1)} - \beta^*  \rVert \\
    \leq \sigma_\mathrm{max}(M_{-}) \left( \lVert \beta^{(k)} - \beta^*  \rVert + \frac{\rho \sigma_\mathrm{max}(M_{-}) }{2} \lVert X^{(k+1)} - X^* \rVert\right) \\
    \leq \sigma_\mathrm{max}(M_{-}) \cdot \\
    \left( \lVert \beta^{(k)} - \beta^*  \rVert + \frac{\rho \sigma_\mathrm{max}(M_{-}) }{2} \left( \lVert X^{(k+1)} - X^* + \frac{1}{\sqrt{2}m_f} D w^{(k+1)} \rVert + \lVert \frac{1}{\sqrt{2}m_f} D w^{(k+1)} \rVert \right) \right) \\
    \leq \sigma_\mathrm{max}(M_{-}) \cdot \\
    \left( \lVert \beta^{(k)} - \beta^*  \rVert + \frac{\rho \sigma_\mathrm{max}(M_{-}) }{2} \left( \sqrt{\frac{1}{m_f} \lVert U^{(k)} - U^* \rVert_G^2 + \frac{1}{2m_f^2} \lVert D w^{(k+1)} \rVert^2} + \lVert \frac{1}{\sqrt{2}m_f} D w^{(k+1)} \rVert \right) \right)  \\
    \leq \sigma_\mathrm{max}(M_{-}) \sqrt{\rho} \sqrt{\lVert U^{(k)} - U^*  \rVert_G^2}\ + \\
    \frac{\rho \sigma_\mathrm{max}^2(M_{-})}{2} \left( \sqrt{\frac{1}{m_f} \lVert U^{(k)} - U^* \rVert_G^2 + \frac{1}{2m_f^2} \lVert D w^{(k+1)} \rVert^2} + \lVert \frac{1}{\sqrt{2}m_f} D w^{(k+1)} \rVert\right).
\end{multline}
\rev{The first step in the reasoning above applies the Euclidian norm and uses the properties of the singular values of a matrix. The second step uses the triangle inequality property of the Euclidian norm, eq. \eqref{eq:lemma2_1002}, and the properties of the singular values of a matrix. The third step uses the triangle inequality as well. The fourth step uses eq. \eqref{eq:lemma2_13}. Finally, the fifth step uses the fact that}
\begin{equation}
    \lVert \beta^{(k)} - \beta^* \rVert = \sqrt{\rho }\sqrt{\frac{1}{\rho}\lVert \beta^{(k)} - \beta^* \rVert^2} \leq \sqrt{\rho}\sqrt{\frac{1}{\rho}\lVert \beta^{(k)} - \beta^* \rVert^2 + \rho \lVert Z^{(k)} - Z^* \rVert^2}, 
\end{equation}
\rev{and the definition of $\lVert U^{(k)} - U^* \rVert_G^2$.}
\vspace{0.3cm}

\subsubsection*{\highlight{Step 3.3: Manipulating the second term on the RHS of eq. (\ref{eq:lemma2_8})}}

For the second term of eq. \eqref{eq:lemma2_8}, we get
\begin{multline}\label{eq:lemma2_10}
     \left(\lVert X^{(k+1)} - X^* \rVert + \lVert \sqrt{2} D w^{(k+1)} \rVert \right)^2 \leq \\
     \left(\lVert X^{(k+1)} - X^* + \frac{1}{\sqrt{2}m_f} D w^{(k+1)}\rVert + \underbrace{\lVert \frac{1}{\sqrt{2}m_f} D w^{(k+1)}\rVert + \lVert \sqrt{2} D w^{(k+1)} \rVert}_{\Bar{w}^{(k+1)}\geq0} \right)^2 \leq\\
     \frac{1}{m_f} \lVert U^{(k)} - U^* \rVert_G^2  + 2 \Bar{w}^{(k+1)} \lVert X^{(k+1)} - X^* + \frac{1}{\sqrt{2}m_f} D w^{(k+1)}\rVert + \left(\Bar{w}^{{(k+1)}}\right)^2 + \frac{1}{2m_f^2} \lVert D w^{(k+1)} \rVert^2 \leq\\
     \frac{1}{m_f} \lVert U^{(k)} - U^* \rVert_G^2  + 2 \Bar{w}^{(k+1)} \sqrt{\frac{1}{m_f} \lVert U^{(k)} - U^* \rVert_G^2 + \frac{1}{2m_f^2} \lVert D w^{(k+1)} \rVert^2} + \left(\Bar{w}^{{(k+1)}}\right)^2 + \frac{1}{2m_f^2} \lVert D w^{(k+1)} \rVert^2,
\end{multline}
\rev{by using the triangle inequality of the Euclidian norm, and then simply developing the square and using eq. \eqref{eq:lemma2_13}.}

\subsection*{\highlight{Step 4:} Simplifying the recursive relationship \eqref{eq:lemma2_8} for $\lVert U^{(k+1)} - U^* \rVert_G^2$}
By replacing the third and second term on the RHS of eq. \eqref{eq:lemma2_8}, with their upper bounds, eq. \eqref{eq:lemma2_9} and eq. \eqref{eq:lemma2_10} respectively, and by combining like terms, we obtain
\begin{multline}\label{eq:lemma2_main}
    \lVert U^{(k+1)} - U^* \rVert_G^2 \leq \underbrace{\frac{2m_f+1}{2m_f(1+\delta)}}_{a} \lVert U^{(k)} - U^* \rVert_G^2 + \\
    \underbrace{\frac{1}{2(1+\delta)}}_{b} \left( 2 \Bar{w}^{(k+1)} \sqrt{\frac{1}{m_f} \lVert U^{(k)} - U^* \rVert_G^2 + \frac{1}{2m_f^2} \lVert D w^{(k+1)} \rVert^2} + \left(\Bar{w}^{{(k+1)}}\right)^2 + \frac{1}{2m_f^2} \lVert D w^{(k+1)} \rVert^2\right) +\\
    \underbrace{\frac{2\sqrt{2}\delta}{(1+\delta)\sigma^2_{\mathrm{min}}(M_{-})}}_{c}\lVert Dw^{(k+1)} \rVert \Bigg( \sigma_\mathrm{max}(M_{-}) \sqrt{\rho} \sqrt{\lVert U^{(k)} - U^*  \rVert_G^2}\ + \\
    \underbrace{\frac{\rho \sigma_\mathrm{max}^2(M_{-})}{2}}_{d} \left( \sqrt{\frac{1}{m_f} \lVert U^{(k)} - U^* \rVert_G^2 + \frac{1}{2m_f^2} \lVert D w^{(k+1)} \rVert^2} + \lVert \frac{1}{\sqrt{2}m_f} D w^{(k+1)} \rVert\right) \Bigg)\\
    - \underbrace{\frac{2\delta}{(1+\delta)\rho \sigma^2_{\mathrm{min}}(M_{-})}}_{e} \lVert D w^{(k+1)} \rVert^2.
\end{multline}
We will now use the fact $\sqrt{a+b}\leq \sqrt{a} +\sqrt{b}$ for $a,b \geq 0$. Eq. (\ref{eq:lemma2_main}) then gives us
\begin{multline}\label{eq:lemma2_16}
    \lVert U^{(k+1)} - U^* \rVert_G^2 \leq a \lVert U^{(k)} - U^* \rVert_G^2 + 2\frac{b}{\sqrt{m_f}} \Bar{w}^{(k+1)} \lVert U^{(k)} - U^* \rVert_G + 2\frac{b}{\sqrt{2}m_f} \Bar{w}^{(k+1)} \lVert D w^{(k+1)} \rVert\\
    + b \left(\Bar{w}^{(k+1)}\right)^2 + \frac{b}{2m_f^2}\lVert D w^{(k+1)} \rVert^2 
    +c \sigma_\mathrm{max}(M_{-}) \sqrt{\rho}\lVert D w^{(k+1)} \rVert \lVert U^{(k)} - U^* \rVert_G \\
    + \frac{cd}{\sqrt{m_f}}\lVert D w^{(k+1)} \rVert \lVert U^{(k)} - U^* \rVert_G +\frac{cd}{\sqrt{2}m_f}\lVert D w^{(k+1)} \rVert^2  +\frac{cd}{\sqrt{2}m_f}\lVert D w^{(k+1)} \rVert^2 - e \lVert D w^{(k+1)} \rVert^2 \\
    = a \lVert U^{(k)} - U^* \rVert_G^2  +\\
    \underbrace{\Bigg( 2\frac{b}{\sqrt{m_f}} \Bar{w}^{(k+1)} +c \sigma_\mathrm{max}(M_{-}) \sqrt{\rho}\lVert D w^{(k+1)} \rVert + \frac{cd}{\sqrt{m_f}}\lVert D w^{(k+1)} \rVert \Bigg)}_{y^{(k+1)}}\lVert U^{(k)} - U^* \rVert_G \\
    + \underbrace{\Bigg( \frac{\sqrt{2}b}{m_f} \Bar{w}^{(k+1)} \lVert D w^{(k+1)} \rVert + b \left(\Bar{w}^{(k+1)}\right)^2 + \frac{b}{2m_f^2}\lVert D w^{(k+1)} \rVert^2 
    +\frac{\sqrt{2}cd}{m_f}\lVert D w^{(k+1)} \rVert^2  - e \lVert D w^{(k+1)} \rVert^2\Bigg)}_{r^{(k+1)}}.
\end{multline}
\rev{By simple substitution of the labeled quantities, we have}
\begin{equation}
    \lVert U^{(k+1)} - U^* \rVert_G^2 \leq a \lVert U^{(k)} - U^* \rVert_G^2 + y^{(k+1)} \sqrt{\lVert U^{(k)} - U^* \rVert_G^2} + r^{(k+1)}.
\end{equation}
An important fact is that $y^{(k+1)}, r^{(k+1)}$ only depend on the noise at iteration $(k+1)$, i.e., these terms are a function of $w^{(k+1)}$.

\subsection*{\highlight{Step 5:} Obtaining bounds for the Wasserstein distances}
We now choose the coupling\footnote{Assume that $x$ and $y$ are two random variables with marginal distributions $p(x)$ and $p(y)$ respectively. Then the coupling between $x$ and $y$ is given by the joint distribution $p(x, y)$, such that $p(x) = \sum_y p(x,y)$ and $p(y)=\sum_x p(x,y)$. The coupling is determined by the conditional distribution $p(x\mid y)$ such that $p(x,y) = p(x\mid y) p(y)$.} between the marginal probability distribution of $Z^{(k)}$ and $Z^*$ (which is seen as a random variable with point mass at the value $Z^*$) to be that for which their normed distance squared is minimized. 
We do the same for the coupling between $\beta^{(k+1)}$ and $\beta^*$ (which is also seen as a random variable with point mass at value $\beta^*$). Using Jensen's inequality for concave functions and the independence between $w^{(k+1)}$ and $U^{(k)}$, we get
\begin{multline}\label{eq:lemma2_14}
    W^2_G(\mu_{U^{(k+1)}}, \mu_{U^*}) \leq \mathbb \lVert U^{(k+1)} - U^* \rVert_G^2 \leq a W^2_G(\mu_{U^{(k)}}, \mu_{U^*}) + \\
    \mathbb{E}\left(y^{(k+1)} \right) \sqrt{W^2_G(\mu_{U^{(k)}}, \mu_{U^*}) } + \mathbb{E}\left( r^{(k+1)} \right) \\
    \leq \left(\sqrt{a} W_G(\mu_{U^{(k)}}, \mu_{U^*}) +\frac{\mathbb{E}\left(y^{(k+1)}\right)}{2\sqrt{a}}\right)^2 + \mathbb{E}\left( r^{(k+1)} \right) - \left( \frac{\mathbb{E}\left(y^{(k+1)}\right)}{2\sqrt{a}}\right)^2.
\end{multline}
We can now bound the Wasserstein distance, by using $\sqrt{a+b}\leq \sqrt{a} +\sqrt{b}$ for $a,b \geq 0$,
\begin{equation}
    W_G(\mu_{U^{(k+1)}}, \mu_{U^*}) \leq \sqrt{a} W_G(\mu_{U^{(k)}}, \mu_{U^*}) +\underbrace{\frac{\mathbb{E}\left(y^{(k+1)}\right)}{2\sqrt{a}} + \sqrt{\mid\mathbb{E}\left( r^{(k+1)} \right) - \left( \frac{\mathbb{E}\left(y^{(k+1)}\right)}{2\sqrt{a}}\right)^2 \mid}}_{h^{(k+1)}}.
\end{equation}
\rev{From eq. (\ref{eq:lemma2_13}) and by applying the same reasoning as before, we obtain}
\begin{equation}\label{eq:lemma2_15}
    W^2(\mu_{X^{(k+1)}}, \mu_{X^* - \frac{1}{\sqrt{2}m_f} D w^{(k+1)}}) \leq \frac{1}{m_f} W^2_G(\mu_{U^{(k)}}, \mu_{U^*}) + \frac{1}{2m_f^2}\mathbb{E}\left( \lVert D w^{(k+1)} \rVert^2\right).
\end{equation}
\rev{By using $\sqrt{a+b}\leq \sqrt{a} +\sqrt{b}$ for $a,b \geq 0$, we finally get}
\begin{equation}\label{eq:lemma2_17}
    W(\mu_{X^{(k+1)}}, \mu_{X^* - \frac{1}{\sqrt{2}m_f} D w^{(k+1)}}) \leq \frac{1}{\sqrt{m_f}} W_G(\mu_{U^{(k)}}, \mu_{U^*}) + \frac{1}{\sqrt{2}m_f}\sqrt{\mathbb{E}\left( \lVert D w^{(k+1)} \rVert^2\right)}.
\end{equation}
The triangle inequality of the $2$-Wasserstein distance yields
\begin{equation}
    W(\mu_{X^{(k+1)}}, \boldsymbol{\mu}^*) \leq W(\mu_{X^{(k+1)}}, \mu_{X^* - \frac{1}{\sqrt{2}m_f} D w^{(k+1)}}) + W( \mu_{X^* - \frac{1}{\sqrt{2}m_f} D w^{(k+1)}},\boldsymbol{\mu}^*),
\end{equation}
and by eq. (\ref{eq:lemma2_14},\ref{eq:lemma2_17}), we obtain the final bound
\begin{multline}
    W(\mu_{X^{(k+1)}}, \boldsymbol{\mu}^*) \leq \frac{1}{\sqrt{m_f}} W_G(\mu_{U^{(k)}}, \mu_{U^*}) + \frac{1}{\sqrt{2}m_f}\sqrt{\mathbb{E}\left( \lVert D w^{(k+1)} \rVert^2\right)} +\\
    W( \mu_{X^* - \frac{1}{\sqrt{2}m_f} D w^{(k+1)}},\boldsymbol{\mu}^*),
\end{multline}
where the last two terms in the RHS are constants.

\section{Proof of Theorem 3}\label{sec:proof_theorem1}
The proof of Theorem 3 is based on two steps: i) we telescopically expand the inequality of eq. (17), and ii) we find sufficient conditions for the telescopic sum to be decreasing with increasing iteration number.
\subsection*{Step 1: Telescopically expanding eq. (13)}
We start from eq. (17) of the main body, which is also given below for convenience,
\begin{equation}\label{eq:theorem1_1}
    W_G(\mu_{U^{(k+1)}}, \mu_{U^*}) \leq \sqrt{a} W_G(\mu_{U^{(k)}}, \mu_{U^*}) +\frac{\mathbb{E}\left(y^{(k+1)}\right)}{2\sqrt{a}} + \sqrt{\mid\mathbb{E}\left( r^{(k+1)} \right) - \left( \frac{\mathbb{E}\left(y^{(k+1)}\right)}{2\sqrt{a}}\right)^2 \mid}.
\end{equation}
We first manipulate the last term on the RHS. By the triangle inequality of the absolute value, we have that 
\begin{equation}\sqrt{\mid\mathbb{E}\left( r^{(k+1)} \right) - \left( \frac{\mathbb{E}\left(y^{(k+1)}\right)}{2\sqrt{a}}\right)^2 \mid} \leq \sqrt{\mid\mathbb{E}\left( r^{(k+1)} \right) \mid + \mid \left( \frac{\mathbb{E}\left(y^{(k+1)}\right)}{2\sqrt{a}}\right)^2 \mid}.
\end{equation}
By the property $\sqrt{a+b}\leq \sqrt{a}+\sqrt{b}$ for $a,b \geq 0$, we further have 
\begin{equation}\sqrt{\mid\mathbb{E}\left( r^{(k+1)} \right) \mid + \mid \left( \frac{\mathbb{E}\left(y^{(k+1)}\right)}{2\sqrt{a}}\right)^2 \mid} \leq \sqrt{\mid\mathbb{E}\left( r^{(k+1)} \right) \mid} + \sqrt{\mid \left( \frac{\mathbb{E}\left(y^{(k+1)}\right)}{2\sqrt{a}}\right)^2 \mid}.
\end{equation}
Therefore eq. (17) can be written as
\begin{equation}\label{eq:theorem1_2}
    W_G(\mu_{U^{(k+1)}}, \mu_{U^*}) \leq \sqrt{a} W_G(\mu_{U^{(k)}}, \mu_{U^*}) +\frac{\mathbb{E}\left(y^{(k+1)}\right)}{\sqrt{a}} + \sqrt{\mid\mathbb{E}\left( r^{(k+1)} \right) \mid},
\end{equation}
since $y^{(k+1)} \geq 0$.
We recursively apply the inequality above to obtain
\begin{multline}\label{eq:theorem1_3}
    W_G(\mu_{U^{(k+1)}}, \mu_{U^*}) \leq \left(\sqrt{a}\right)^{k+1} W_G(\mu_{U^{0}}, \mu_{U^*}) +\\
    \sum_{l=1}^{k+1} \left(\sqrt{a}\right)^{k-l}\mathbb{E}\left(y^{(l)} \right) +
    \sum_{l=1}^{k+1} \left(\sqrt{a}\right)^{k+1-l} \sqrt{\mid\mathbb{E}\left( r^{(l)} \right) \mid}.
\end{multline}
From eq. (18-21), we note that for given $a,b,c,d$, and $e$, the terms $\mathbb{E}\left(y^{(l)} \right)$ and $\mathbb{E}\left( r^{(l)} \right)$ are bounded, as they only depend on the noise at iteration $l$ and the Euclidian norm, as well as the square of the Euclidian norm, of a Gaussian random variable are bounded. Assume the terms $\mathbb{E}\left(y^{(l)} \right)$ and $\mathbb{E}\left( r^{(l)} \right)$ are bounded by $Y\geq 0$ and $R \geq 0$ respectively. Then, eq. \eqref{eq:theorem1_3} becomes
\begin{equation}\label{eq:theorem1_4}
    W_G(\mu_{U^{(k+1)}}, \mu_{U^*}) \leq \left(\sqrt{a}\right)^{k+1} W_G(\mu_{U^{0}}, \mu_{U^*}) +
    \sum_{l=1}^{k+1} \left(\sqrt{a}\right)^{k-l}Y +
    \sum_{l=1}^{k+1} \left(\sqrt{a}\right)^{k+1-l} \sqrt{R}.
\end{equation}
Combining eq. \eqref{eq:theorem1_4} with eq. (16) in the main body, we obtain
\begin{multline}
        W(\mu_{X^{(k+1)}}, \boldsymbol{\mu}^*) \leq \frac{1}{\sqrt{m_f}} \left(\sqrt{a}\right)^{k} W_G(\mu_{U^{0}}, \mu_{U^*}) +\\
    \frac{1}{\sqrt{m_f}}\sum_{l=1}^{k} \left(\sqrt{a}\right)^{k-l}Y +
    \frac{1}{\sqrt{m_f}}\sum_{l=1}^{k} \left(\sqrt{a}\right)^{k-l} \sqrt{R} + \\
        \frac{1}{\sqrt{2}m_f}\sqrt{\mathbb{E}\left( \lVert D w^{(k+1)} \rVert^2\right)} +
    W( \mu_{X^* - \frac{1}{\sqrt{2}m_f} D w^{(k+1)}},\boldsymbol{\mu}^*),
\end{multline}
where the last two terms on the RHS are constants. Assuming that $a < 1$, then, since $a>0$, we have $\sum_{l=0}^\infty a^k = \frac{1}{1-a}$, which leads to the inequality
\begin{multline}\label{eq:theorem1_5}
        W(\mu_{X^{(k+1)}}, \boldsymbol{\mu}^*) \leq \frac{1}{\sqrt{m_f}} \left(\sqrt{a}\right)^{k} W_G(\mu_{U^{0}}, \mu_{U^*}) +\\
    \frac{1}{\sqrt{a m_f}}\frac{Y}{1-\sqrt{a}} +
    \frac{1}{\sqrt{m_f}}\frac{\sqrt{R}}{1-\sqrt{a}}+ \\
        \frac{1}{\sqrt{2}m_f}\sqrt{\mathbb{E}\left( \lVert D w^{(k+1)} \rVert^2\right)} +
    W( \mu_{X^* - \frac{1}{\sqrt{2}m_f} D w^{(k+1)}},\boldsymbol{\mu}^*).
\end{multline}
From the last equation, we observe that if $a < 1$, the upper bound for the Wasserstein distance $W(\mu_{X^{(k+1)}}, \boldsymbol{\mu}^*)$ decreases as the iteration number increases. This is because the first term on the RHS of eq. \eqref{eq:theorem1_5} decreases as $k$ increases, while the remaining terms on the RHS of eq. \eqref{eq:theorem1_5} are constants for all iterations $k \geq 0$.
\subsection*{Step 2: Finding sufficient conditions for $a<1$}
We start with the definition of $\delta$ from eq. (20) of the main body. We observe that $\delta$ is a function of both $\kappa$ and $\rho$, given by
\begin{equation}\label{eq:delta}
    \delta(\kappa, \rho) = \min \Biggl \{ \frac{(\kappa -1) \sigma^2_{\mathrm{min}}(M_{-})}{\kappa \sigma^2_{\mathrm{max}}(M_{+})}, \frac{m_f}{\frac{\rho}{4} \sigma^2_{\mathrm{max}}(M_{+}) + \frac{\kappa M_f^2}{\rho \sigma^2_{\mathrm{min}}(M_{-})} }\Biggr \}.
\end{equation}
We observe that only the second argument in the definition of $\delta(\kappa, \rho)$ depends on $\rho$. Assuming a $\kappa>1$ is selected, then the value 
\begin{equation}\label{eq:theorem1_6}
    \rho(\kappa) = \dfrac{2\kappa^{\frac{1}{2}}M_f}{\sigma_\mathrm{min}(M_{-})\sigma_\mathrm{max}(M_{+})}
\end{equation}
maximizes the second term in the $\min$ of eq. \eqref{eq:delta} and therefore $\delta$ for the given $\kappa$. In other words, $\rho(\kappa)$ from eq. \eqref{eq:theorem1_6} maximizes $\delta(\kappa, \rho)$ for the given $\kappa$. The maximum $\delta$ as a function of $\kappa$, termed $\delta_\mathrm{max}(\kappa)$, is hence given by
\begin{equation}\label{eq:theorem1_7}
    \delta_\mathrm{max}(\kappa) = \min \Biggl \{ \frac{(\kappa -1) \sigma^2_{\mathrm{min}}(M_{-})}{\kappa \sigma^2_{\mathrm{max}}(M_{+})}, \frac{m_f \sigma_\mathrm{min}(M_{-})}{\kappa^{\frac{1}{2}}M_f \sigma_\mathrm{max}(M_{+})}\Biggr \}.
\end{equation}
In order to find the maximum $\delta(\kappa, \rho)$ we need to maximize $\delta_\mathrm{max}(\kappa)$ in eq. \eqref{eq:theorem1_7} with respect to $\kappa$. We observe that the first term in eq. \eqref{eq:theorem1_7} is monotonically increasing as a function $\kappa$, while the second term in eq. \eqref{eq:theorem1_7} is monotonically decreasing as a function $\kappa$. Therefore, to obtain the maximum $\delta$, we choose $\kappa$ such that the two terms are equal. Such a $\kappa$ exists and comes out to be
\begin{equation}\label{eq:theorem1_8}
    \kappa = 1+ \frac{1}{2}\sqrt{4 \frac{\tau_G^2}{\tau_f^2} + \frac{\tau_G^4}{\tau_f^4}} + \frac{\tau_G^2}{2\tau_f^2} > 1,
\end{equation}
where 
\begin{equation}
    \tau_G = \frac{\sigma_\mathrm{max}(M_{+})}{\sigma_\mathrm{min}(M_{-})},\ \tau_f = \frac{M_f}{m_f}.
\end{equation}
By plugging in $\kappa$ from eq. \eqref{eq:theorem1_8} to eq. \eqref{eq:theorem1_7}, we get the maximum possible value of $\delta(\kappa, \rho)$ to be
\begin{equation}
    \delta_\mathrm{max} = \frac{1}{2\tau_f}\sqrt{\frac{1}{\tau_f^2} + \frac{4}{\tau_G^2}} - \frac{1}{2\tau_f^2}.
\end{equation}
We turn our attention to the definition of $a$ in eq. (21). $a <1$ if and only if $2m_f \delta >1$. Therefore, for convergence we need
\begin{equation}
    \delta_\mathrm{max} > \frac{1}{2m_f}.
\end{equation}
A sufficient condition for convergence is thus the following
\begin{equation}
    \frac{m_f}{M_f}\sqrt{\frac{m_f^2}{M_f^2} + \frac{4\sigma_\mathrm{min}^2(M_{-})}{\sigma_\mathrm{max}^2(M_{+})}} - \frac{m_f^2}{M_f^2} > \frac{1}{m_f}.
\end{equation}
The last equation can equivalently be written as
\begin{equation}
    \tau_f^{-1}\sqrt{\tau_f^{-2}+4\tau_G^{-2}}-\tau_f^{-2} > m_f^{-1}.
\end{equation}

\bibliography{sample}

\end{document}